\documentclass[11pt]{article}

\usepackage{dsfont}

\def\colorful{0}

\usepackage{amsmath,amsthm,amsfonts,amssymb,epsfig,color,float,graphicx,verbatim, enumitem}
\usepackage{listings}
\usepackage{multirow}
\usepackage{algorithm}
\usepackage{algpseudocode}
\usepackage{bbm}

\newif\ifhyper\IfFileExists{hyperref.sty}{\hypertrue}{\hyperfalse}
\hypertrue
\ifhyper\usepackage{hyperref}\fi

\usepackage[nameinlink]{cleveref}

\usepackage{enumitem}

\usepackage{framed}
\usepackage{nicefrac}

\def\nnewcolor{1}
\ifnum\nnewcolor=1

\fi
\ifnum\nnewcolor=0

\fi

\ifnum\colorful=1

\else

\fi

\newtheorem{theorem}{Theorem}[section]

\newtheorem{lemma}[theorem]{Lemma}
\newtheorem{informal theorem}[theorem]{Theorem (informal statement)}

\newtheorem{fact}[theorem]{Fact}

\theoremstyle{definition}
\newtheorem{definition}[theorem]{Definition}
\newcommand{\eqdef}{\stackrel{{\mathrm {\footnotesize def}}}{=}}

\newcommand{\bx}{\mathbf{x}}

\newcommand{\bw}{\mathbf{w}}

\newcommand{\err}{\mathrm{err}}

\newcommand{\R}{\mathbb{R}}

\newcommand{\N}{\mathbb{N}}
\newcommand{\E}{\mathbf{E}}

\newcommand{\eps}{\epsilon}

\newcommand{\pr}{\mathbf{Pr}}

\newcommand{\poly}{\mathrm{poly}}
\newcommand{\polylog}{\mathrm{polylog}}

\newcommand{\sign}{\mathrm{sign}}

\newcommand{\opt}{\mathrm{OPT}}

\newcommand{\Ind}{\mathds{1}}

\newcommand{\guess}{\mathrm{guess}}
\newcommand{\C}{\mathcal{C}}
\newcommand{\A}{\mathcal{A}}
\newcommand{\heavy}{\mathrm{heavy}}
\newcommand{\light}{\mathrm{light}}
\newcommand{\stat}{\mathrm{STAT}}

\newcommand{\algparbox}[1]{\parbox[t]{\dimexpr\linewidth-\algorithmicindent}{#1\strut}}

\title{Faster Algorithms for Agnostically Learning Disjunctions\\ 
and their Implications}

\author{
Ilias Diakonikolas\thanks{Supported in part by NSF Medium Award CCF-2107079 and an H.I. Romnes Faculty Fellowship.
}\\
University of Wisconsin-Madison\\
{\tt ilias@cs.wisc.edu}\\
\and
Daniel M. Kane\thanks{Supported by NSF Medium Award CCF-2107547, and
NSF Award CCF-1553288 (CAREER).}\\
University of California, San Diego\\
{\tt dakane@cs.ucsd.edu}\\
\and
Lisheng Ren\thanks{Supported in part by NSF Medium Award CCF-2107079.}\\
University of Wisconsin-Madison\\
{\tt lren29@wisc.edu}\\
\and
}
\begin{document}

\maketitle

\begin{abstract}
We study the algorithmic task of learning Boolean disjunctions 
in the distribution-free agnostic PAC model. The best known 
agnostic learner for the class of disjunctions over $\{0, 1\}^n$ 
is the $L_1$-polynomial regression algorithm, 
achieving complexity $2^{\tilde{O}(n^{1/2})}$. 
This complexity bound is known to be nearly best possible 
within the class of Correlational Statistical Query (CSQ) algorithms. 
In this work, we develop an agnostic learner 
for this concept class with complexity $2^{\tilde{O}(n^{1/3})}$. 
Our algorithm can be implemented in the Statistical Query (SQ) model, 
providing the first separation between the SQ and CSQ models in 
distribution-free agnostic learning.
\end{abstract}

\setcounter{page}{0}
\thispagestyle{empty}
\newpage

\section{Introduction}
A disjunction (resp. conjunction) 
over $\{0, 1\}^n$ is an OR (resp. AND) of literals, where a literal is either a 
Boolean variable or its negation. While disjunctions are known to be 
efficiently learnable in Valiant's realizable PAC model~\cite{val84} 
(i.e., in the presence of clean/consistent labels), 
the learning task becomes substantially more challenging 
in the presence of partially corrupted labels. 
Here we study the task of learning disjunctions (or equivalently conjunctions) 
in the distribution-free agnostic PAC model \cite{Haussler:92, KSS:94}.
Agnostically learning disjunctions 
was one of the original problems  
studied by~\cite{KSS:94}, 
and has since been highlighted in Avrim Blum's 
FOCS 2003 tutorial~\cite{Blum03}.

In the agnostic model, no assumptions are made about the labels, and 
the goal of the learner is to compute a hypothesis that is competitive 
with the best-fit function in the target class.
For concreteness, we formally define the agnostic model 
in the Boolean setting below.

\begin{definition} [Distribution-free Agnostic PAC learning] 
\label{def:agnostic-learning}
Let $\C$ be a concept class of functions 
$f: \{0, 1\}^n \to \{0, 1\}$ and $D$ be fixed but unknown 
distribution of $(\bx,y)$ over $\{0,1\}^n\times \{0,1\}$. 
Given an error parameter $\eps\in(0,1)$ and sample access 
to $D$, the goal of an agnostic PAC learner 
$\A$ is to output a hypothesis $h: \{0, 1\}^n \to \{0, 1\}$ 
such that with high probability  
$\pr_{(\bx,y)\sim D}[h(\bx)\neq y]\leq \opt+\eps \;,$
where $\opt\eqdef \min_{f\in \C}\pr_{(\bx,y)\sim D}[f(\bx)\neq y]$.
We say that $\A$ 
agnostically PAC learns $\C$ to error $\eps$.    
\end{definition}

\noindent Prior to this work, 
the fastest---and essentially only known non-trivial---
algorithm for agnostically learning disjunctions was the $L_1$-polynomial 
regression algorithm~\cite{KKMS08}. As shown in that work, 
the $L_1$-regression algorithm agnostically 
learns disjunctions over $\{0, 1\}^n$ up to excess error $\eps$ 
with sample and computational complexity bounded above by 
$2^{\tilde{O}(n^{1/2} \, \log(1/\eps))}$\footnote{Throughout this paper, we will 
assume that the failure probability $\delta$ is a small universal constant, 
e.g., $\delta = 1/10$. Standard arguments can boost this to any desired $\delta$ 
with only a $\polylog(1/\delta)$ complexity blowup.}. This complexity upper 
bound is tight, as a function of $n$, for the $L_1$-regression algorithm. 

In terms of computational limitations, 
it is known that $2^{\Omega(n^{1/2})}$ is a  
complexity lower bound in various restricted models of computation, 
including Perceptron-based approaches~\cite{KlivansS10} 
and Correlational Statistical Query (CSQ) algorithms~\cite{GKK20}.  
In the (more general) Statistical Query (SQ) model, 
to the best of our knowledge, the strongest known 
hardness result is a quasi-polynomial SQ 
lower bound~\cite{Feldman:09} that applies 
even under the uniform distribution. 
Finally, we note that \cite{FGRW09} proved strong NP-hardness results 
for agnostically learning disjunctions with a halfspace hypothesis. 
This result does not rule out efficient improper learning.

The vast gap between known upper and lower bounds motivates further algorithmic investigation of this fundamental learning task.
In this work, we give a new algorithm for agnostically learning disjunctions with substantially improved complexity. Specifically, we show the following:
\begin{theorem}[Main Result] \label{thm:main-intro-informal}
There exists an algorithm that agnostically PAC 
learns the class of disjunctions  
over $\{0, 1\}^{n}$ to error $\eps$ 
with sample and computational complexity 
$2^{\Tilde{O}(n^{1/3}\log(1/\eps))}$.
\end{theorem}

\noindent We give two algorithms that achieve the above guarantee 
(up to the $\tilde{O}()$ factor in the exponent). 
Our first algorithm is simpler and is slightly more efficient 
(with better logarithmic factors in the $\tilde{O}()$). 
Our second algorithm is implementable in the SQ model 
with complexity $2^{\Tilde{O}(n^{1/3}\log(1/\eps))}$. 

As a corollary, we obtain a super-polynomial separation between 
the CSQ and SQ models in the context of agnostic learning, 
answering an open problem posed in~\cite{GKK20}.
We elaborate on this connection in the proceeding discussion.

\paragraph{Discussion} The CSQ model~\cite{BshoutyFeldman:02} 
is a subset of the SQ model~\cite{Kearns:98}, where the oracle access
is of a special form (see \Cref{def:sq} and Appendix~\ref{sec:csq-lb}).
In the context of learning Boolean-valued functions, the two models are 
known to be equivalent in the distribution-specific setting 
(i.e., when the marginal distribution on feature vectors 
is known to the learner); see~\cite{BshoutyFeldman:02}. 
However, they are not in general 
equivalent in the distribution-free PAC model. 
In the realizable PAC setting, there are known 
natural separations between the CSQ and SQ models. 
Notably,~\cite{Feldman11} showed that Boolean halfspaces 
are not efficiently CSQ learnable (even though they are efficiently SQ learnable). 

In the agnostic PAC model studied here, we are not aware of a natural concept class separating the two models. Our algorithm provides such a natural separation for the class of disjunctions. In more detail,~\cite{GKK20} asks:

\begin{quote}
``Our CSQ lower bounds do not readily extend to the general SQ model, and a very compelling direction for future work is to investigate whether such an extension is possible''. 
\end{quote}
We answer this question in the negative by establishing a 
super-polynomial separation between the CSQ and SQ models  
for agnostically learning one of the most basic concept classes.  

It is worth pointing out that the 
$L_1$-regression algorithm is known to be 
implementable in the SQ model (but not in CSQ). 
We also point out 
(see \Cref{fact:csq-eps-weak-learning}) 
that there exists a weak agnostic CSQ 
learner for disjunctions with complexity 
$2^{\tilde{O}(n^{1/2})}$.

A conceptual implication is that there exists a qualitative 
difference between distribution-specific and distribution-free agnostic 
learning. In the distribution-specific setting, in particular when the 
underlying feature distribution is a discrete product distribution 
or the standard Gaussian, prior work~\cite{DFTWW15, DKPZ21} 
has shown that the $L_1$-polynomial regression algorithm 
is optimal in the SQ model---that is, CSQ and SQ 
are polynomially equivalent. 

\medskip

Finally, by building on the agnostic learner 
of \Cref{thm:main-intro-informal}, we also obtain an $\alpha$-approximate agnostic learner, i.e., an algorithm with error guarantee 
of $\alpha \cdot \opt+\eps$ for some $\alpha \geq 1$, 
with complexity  
$2^{\Tilde{O}(n^{1/3}\alpha^{-2/3})}\poly(1/\eps)$ 
(see \Cref{thm:main-tradeoff}).

\subsection{Technical Overview} \label{ssec:techniques}

In this section, we summarize the key ideas of our algorithms.

\vspace{-0.2cm}

\paragraph{Sample-Based Agnostic Learner}
We start by describing our first agnostic learner with complexity 
$2^{\Tilde{O}(n^{1/3}\log(1/\eps))}$. 
We start by recalling the $L_1$-polynomial regression 
algorithm~\cite{KKMS08}. In particular, the $L_1$ regression algorithm allows one to agnostically learn a function $f$ to error $\opt+\eps$ in time $n^{O(d)}$ if $f$ can be $\eps$-approximated, in $L_\infty$-norm, 
by polynomials of degree $d$. The work of \cite{KKMS08} 
shows that, since disjunctions can be approximated by polynomials 
of degree $O(\sqrt{n})$ (see~\cite{NisanSzegedy:94, Paturi:92}), 
then the $L_1$-algorithm is an agnostic learner for disjunctions
with sample size and running time roughly $n^{O(\sqrt{n})}$. 
Since the $O(\sqrt{n})$-degree bound for polynomial approximation
to the class of disjunctions is tight, this complexity upper bound
is best possible for $L_1$-regression. 

For the rest of this section, we will focus on the special case
of agnostically learning monotone disjunctions (as the general task 
can be easily reduced to it,  by adding new coordinates 
for the negations of each of the original coordinates). 
The starting point of our algorithm is the following
observation: using the $L_1$-polynomial approximation approach, 
we can show that monotone disjunctions can be approximated 
on all Hamming-weight at most $r$ strings---where $r$ is a parameter 
to be optimized in hindsight---using polynomials 
of degree $O(\sqrt{r})$ (see Lemma~\ref{lem:approx-degree-coordinate}).
This means that if the underlying distribution $D_X$ on the domain $\{0, 1\}^n$
is largely supported on ``low-weight'' strings, 
we obtain a more efficient agnostic learning algorithm. 

If this is not true (i.e., if $D_X$ assigns significant probability 
mass to strings of high Hamming weight), 
we consider two separate cases:
\begin{enumerate}[leftmargin=*]
\item If the optimal conjunction, $f^\ast$, has $f^\ast(\bx)=1$ on almost all of the high weight inputs, we can return a hypothesis that returns $1$ on the high weight inputs and uses the $L_1$ regression algorithm to learn a nearly optimal classifier on the low weight inputs.
\item If the optimal conjunction, $f^\ast$, assigns $f^\ast(\bx)=0$ to a reasonable fraction of high weight inputs, our algorithm can guess a specific high weight input for which $f^\ast(\bx)=0.$ If we guessed correctly, we know that none of the $1$-entries in $\bx$ can be in the support of $f^\ast$, allowing us to throw away these $r$ coordinates and recurse on a substantially smaller problem.
\end{enumerate}
The algorithm iteratively removes coordinates using Case 2 until it eventually ends up in Case 1.
Note that we can only land in Case 2 at most $n/r$ times 
before the input size becomes trivial. 
Each reduction requires that we guess a positive, high-weight input, 
which---unless we are in Case 1--- 
will constitute at least an $\eps$-fraction of the probability mass 
of the high-weight inputs. 
Thus, the probability that we guess correctly 
enough times will be roughly $\eps^{n/r}$. 
In other words, if we attempt this scheme
$(1/\eps)^{n/r}$ times, we expect that at least one attempt will succeed. 
This gives us a final algorithm with sample and computational complexity roughly $(1/\eps)^{n/r} n^{O(\sqrt{r})}$. 
Setting $r$ to be roughly $n^{2/3}$ 
gives us a complexity upper bound of $2^{\tilde{O}(n^{1/3})}$.

\vspace{-0.2cm}

\paragraph{SQ Agnostic Learner}
Note that the algorithm described above
is not efficiently implementable in the SQ model 
as it requires picking out particular high-weight inputs. 
We next show that there exists an SQ algorithm 
that works along similar lines and requires comparable resources. 
We start by pointing out that the $L_1$-regression 
part of the previous algorithm is known to be 
implementable in the SQ model (see, e.g.,~\cite{DFTWW15}). 
We can thus use this to perform agnostic learning on low-weight inputs. 
To deal with the high-weight inputs, 
we instead consider {\em heavy} coordinates. 
A coordinate is termed {\em heavy} if it shows up 
in more than roughly an $r/n$-fraction of inputs. 
Observe that once we have reduced 
to only high-weight inputs, 
there must be some heavy coordinates. 
For these coordinates, we again have two cases:
\begin{enumerate}[leftmargin=*]
\item None of the heavy coordinates are in the support of our 
target disjunction. In this case, we can remove all such 
coordinates from our domain, and we will be left with many low-weight inputs 
on which we can agnostically learn the function as before. 
\item There is some heavy coordinate in the support 
of our target disjunction. In this case, if we guess such a 
coordinate, we learn an $r/n$-fraction of inputs 
on which the true function must be true.
\end{enumerate}
Overall, by guessing that we are either in Case 1 
or guessing the correct coordinate in Case 2 $n/r$ times, 
we can (if we guess correctly) learn the value 
of the target function on a constant fraction of inputs. 
The probability of success of these guesses 
is roughly $1/n^{n/r}$, 
so we obtain complexity of roughly $n^{n/r} \, n^{O(\sqrt{r})}$.
By selecting $r = n^{2/3}$, this gives a final complexity 
bound of roughly $2^{\tilde{O}(n^{1/3})}$.
It is straightforward to check that this algorithm is implementable in 
the SQ model with comparable complexity.

Interestingly, although we have given 
an SQ algorithm with complexity $2^{\tilde{O}(n^{1/3})}$, 
this complexity bound is not possible for CSQ algorithms. 
In particular, it can be shown that any CSQ algorithm 
requires either roughly $2^{\Omega(n^{1/2})}$ correlational 
queries or queries of accuracy better than roughly $2^{n^{-1/2}}$. 
This follows by combining known results. 
In particular,~\cite{NisanSzegedy:94} showed that 
the approximation degree of the class of disjunctions 
is $\Omega(\sqrt{n})$, and a result of~\cite{GKK20} 
shows that this implies a $2^{\Omega(\sqrt{n})}$ CSQ 
lower bound. 
This is essentially because LP duality implies that high approximation degree 
allows one to construct a moment-matching construction,  
which when embedded among a random subset of coordinates, 
is hard for a CSQ algorithm to learn.

\vspace{-0.2cm}

\paragraph{Approximate Agnostic Learner}
We also explore the complexity of agnostic learning with approximate error 
guarantees, obtaining a time-accuracy tradeoff. 
In particular, if we only require our algorithm 
to obtain accuracy $\alpha \cdot \opt$, for some $\alpha > 1$, 
it is sufficient to obtain a weak learner that does slightly
better than 50\% when $\opt < 1/\alpha$. 
This means that we will only need to guess correctly 
in Case 2 roughly $n/(r\alpha)$ times. Furthermore, as our algorithm only needs to succeed when $\opt < 1/\alpha$, the $L_1$ regression algorithm only needs to consider polynomials of degree $O(\sqrt{r/\alpha})$.
This gives a final runtime of roughly 
$n^{n/(r \alpha)} n^{O(\sqrt{r/\alpha})}$. 
Optimizing $r$ to be $n^{1/3} \alpha^{-1/3}$, 
we get an algorithm with runtime roughly $2^{O(n^{1/3}\alpha^{-2/3})}$.

\paragraph{Organization} 
After basic background in \Cref{sec:prelims}, 
in \Cref{sec:non-sq-ub} we give our main algorithm for agnostically 
learning disjunctions. In \Cref{sec:sq-ub}, we give an alternative SQ 
agnostic learner with qualitatively the same complexity. Finally, 
\Cref{sec:tradeoff-alg}, gives our approximate agnostic learner. 
Some of the proofs and technical details 
have been deferred to an Appendix. 

\section{Preliminaries} \label{sec:prelims}

\paragraph{Notation}
We use $\{0, 1\}$ for Boolean values.
For $n\in \N$, we let $[n] \eqdef \{i\in \N\mid i\leq n\}$.
For a finite set $S$, 
we use $u(S)$ to denote the uniform distribution over all elements in $S$.
For $\bx\in \{0, 1\}^n$, we use $W(\bx)$ to denote the Hamming 
weight of $\bx$, defined as 
$W(\bx)  \eqdef \sum_{i\in [n]} \bx_i$. 
We define the Hamming weight of $\bx$ 
on a subset of coordinates $I\subseteq [n]$ 
as $W_I(\bx)\eqdef \sum_{i\in I}\bx_i$.
A monotone disjunction is any Boolean function 
$f:\{0, 1\}^n\to \{0, 1\}$ of the form 
$f(\bx)=\bigvee_{i\in S} \bx_i$, 
where $S\subseteq [n]$ is the  
set of relevant variables.

\paragraph{Probability Basics}
We will need the following well-known fact about uniform convergence of empirical processes.
We start by recalling the definition of VC dimension. 

\begin{definition}[VC-Dimension] \label{def:vc-dimension}
For a class $\C$ of Boolean functions $f:X\to \{0,1\}$,
the VC-dimension of $\C$ is the largest $d$ such that there exist $d$ points $\bx_1,\bx_2,\cdots,\bx_d\in X$ so that for any Boolean function $g:\{\bx_1,\bx_2,\cdots,\bx_d\}\to\{0,1\}$, there exists an $f\in \C$ satisfying $f(\bx_i)=g(\bx_i)$, for all $1\leq i\leq d$.
\end{definition} 
Then the VC inequality is the following:
\begin{fact} [VC-Inequality] \label{fct:vc-inequality}
    Let $\C$ be a class of Boolean functions on $X$ with VC-dimension $d$, and let $D$ be a distribution on $X$.
    Let $\eps>0$ and let $n$ be an integer at least a sufficiently large constant multiple of $d/\eps^2$. Then, if $\bx_1,\bx_2,\cdots,\bx_n$ are i.i.d.\ samples from $D$, we have that:
    \[
    \pr\left [\sup_{f\in \C}\left |\frac{\sum_{j=1}^n f(\bx_j)}{n}-\E_{\bx\sim D}[f(\bx)]\right |\geq \eps\right ]=\exp(-\Omega(n\eps^2))\; .
    \]
\end{fact}

\paragraph{Approximate Degree and $L_1$-Regression}
We will need the following definitions and facts about 
approximate degree and polynomial $L_1$-regression.

\begin{definition} [Approximate degree]
Let $f:X\to \{0, 1\}$ be a Boolean-valued function, 
where $X$ is a finite subset of $\R^n$. 
The $\eps$-approximate degree $\deg_{\eps,X}(f)$ of $f$ on $X$, 
$0< \eps <1$,
is the least degree of a polynomial $p:X \to \R$ 
such that $|f(\bx)-p(\bx)|\leq \eps$ for all $\bx\in X$. 
For a class $\C$ of Boolean functions, 
we define the $\eps$-approximate 
degree of $\C$ on $X$ as $\deg_{\eps,X}(\C) \eqdef \max_{f\in \C} \deg_{\eps,X}(f)$. 
\end{definition}

\noindent The main known technique for agnostic 
learning is the $L_1$ polynomial regression algorithm \cite{KKMS08}. 
This algorithm uses linear programming 
to compute a low-degree polynomial that minimizes the $L_1$-distance to the target function. 
Its performance hinges on how well the underlying concept class $\C$ can be approximated, in $L_1$ norm, 
by a low-degree polynomial. In more detail, 
if $d$ is the (minimum) degree such that any $f\in \C$ 
can be $\eps$-approximated in $L_1$ norm by a degree-$d$ 
polynomial, the algorithm has sample and computational 
complexity $n^{O(d)}/\poly(\eps)$, where $\eps$ is the excess error. 
It is also well-known that this algorithm can be implemented in the SQ model; see, e.g.,~\cite{DFTWW15}
Our algorithm will use $L_1$-polynomial regression as a subroutine.

\begin{fact}[\cite{KKMS08}] \label{fct:l1-regression}
Let $\C$ be a concept class of functions 
$f:X\to \{0, 1\}$, where $X$ is a finite subset of $\R^n$. There is a degree-$O(\deg_{\eps,X}(\C))$ polynomial 
$L_1$-regression algorithm that distribution-free agnostically learns $\C$ to additive error $\eps$ 
and has sample and computational complexity $n^{O(\deg_{\eps,X}(\C))}$.
Furthermore, the $L_1$-regression algorithm 
can be implemented in the SQ model 
with the same complexity, i.e., having $T$ time complexity and using $q$ queries to $\stat(\tau)$, where $\max(T,q,1/\tau)=n^{O(\deg_{\eps,X}(\C))}$.
\end{fact}

\section{Sample-based Agnostic Learner} \label{sec:non-sq-ub}

In this section, we give an agnostic learner for disjunctions with complexity 
$2^{\Tilde{O}(n^{1/3}\log(1/\eps))}$, thereby establishing 
Theorem~\ref{thm:main-intro-informal}. The agnostic 
learner presented here makes essential use of the samples. 
An SQ agnostic learner with similar complexity is given in 
the Section~\ref{sec:sq-ub}.

We start by pointing out two simplifications 
that can be made without 
loss of generality. First, it suffices to consider 
{\em monotone} disjunctions. As is well-known, 
one can easily and efficiently reduce the general task 
to the task of agnostically learning 
monotone disjunctions by including 
negated variables as additional features.
Second, it suffices to develop a {\em weak} agnostic learner with the desired complexity. In our context, 
a weak agnostic learner is an algorithm 
whose output hypothesis performs slightly better 
than a random guess, when such a hypothesis in $\C$ exists. 
Given such an algorithm, we can leverage standard agnostic 
boosting techniques to obtain a strong agnostic learner, i.e., an algorithm with accuracy $\opt+\eps$, 
with qualitatively the same complexity 
(up to a polynomial factor). 
Specifically, it suffices to establish the following result:

\begin{theorem}[Weak Agnostic Learner for Monotone Disjunctions] \label{thm:main-weak-non-sq}
    Let $D$ be an unknown distribution supported on 
    $\{0, 1\}^{n}\times \{0, 1\}$ and $\eps\in (0,1/2)$. 
    Suppose there is a monotone disjunction 
    $f: \{0, 1\}^{n}\to \{0, 1\}$ such that 
    $\pr_{(\bx,y)\sim D}[f(\bx)\neq y]\leq\ 1/2-\eps$.
    Then there is an algorithm that given i.i.d.\ sample access to $D$ 
    and $\eps$, has sample and computational complexity 
    $2^{\Tilde{O}(n^{1/3}\log(1/\eps))}$, 
    and with probability at least $2^{-{O}(n^{1/3}\log(1/\eps))}$ 
    returns a hypothesis $h:\{0, 1\}^{n}\to \{0, 1\}$ such that $\pr_{(\bx,y)\sim D}[h(\bx)\neq y]\leq 1/2-\Omega(\eps)$. 
\end{theorem}

\noindent Note that by repeating the algorithm of 
Theorem~\ref{thm:main-weak-non-sq} $2^{O(n^{1/3}\log (1/\eps))}$ times 
and testing the empirical error of the output hypothesis each time, 
we get a weak agnostic learner that succeeds with at least constant probability. We show how to apply standard boosting tools,  
in order to obtain Theorem~\ref{thm:main-intro-informal}, in Appendix~\ref{app:exact-nonsq}.

In the body of this section, we proceed to describe our weak 
agnostic learner, thereby proving Theorem~\ref{thm:main-weak-non-sq}.
Let $f_S(\bx)=\bigvee_{i\in S}\bx_i$ be an arbitrary 
monotone disjunction with optimal loss, 
i.e., $\pr_{(\bx,y)\sim D}[f_S(\bx)\neq y] = \opt$. 
We set the radius parameter $r \eqdef n^{2/3}$ 
and partition the domain into two sets:  
$X_\light\eqdef\{\bx\in \{0, 1\}^n\mid W(\bx)\leq r\}$, 
the set of 
of ``low'' Hamming weight strings; 
and  $X_\heavy\eqdef\{\bx\in \{0, 1\}^n\mid W(\bx)> r\}$,  
the set of ``high'' Hamming weight strings.

Since points in $X_\light$ have Hamming weight at most 
$r$, one can leverage the properties of Chebyshev polynomials 
to construct $\eps$-approximate polynomials 
for monotone disjunctions on $X_\light$ of degree
$O(r^{1/2}\log(1/\eps))=O(n^{1/3}\log(1/\eps))$. 
This is shown in the lemma below.

\begin{lemma}[Approximate Degree on Hamming weight $\leq r$ Strings] \label{lem:approx-degree-coordinate} 
    Let $X\subseteq \{0, 1\}^n$, $I\subseteq [n]$, 
    and $\C$ be the concept class containing 
    all monotone disjunctions of the form 
    $f_S(\bx)=\bigvee_{i\in S}\bx_i$ for $S\subseteq I$ 
    and the constant function $f(\bx) \equiv 1$.
    Let $r \eqdef \max_{\bx\in X}W_I(\bx)$. Then 
    \[
    \deg_{\eps,X}(\C)=\begin{cases}
			O(r^{1/2}(1-2\eps)^{1/2}),  & \;\eps\in [1/4,1/2)\; ;\\
            O(r^{1/2}\log(1/\eps)),    & \;\eps\in (0,1/4)\; .
		 \end{cases}
    \]
\end{lemma}
\begin{proof}
    First notice that the constant function 
    $f(\bx) \equiv 1$ can be approximated by a polynomial of degree-$0$ with $0$ error.  
    Let $f_S(\bx)=\bigvee_{i\in S} \bx_i$ be the target monotone 
    disjunction to be approximated by a polynomial.
    Note that $f_S(\bx)=1$ if $W_S(\bx)>0$; and $0$ otherwise.
    We leverage the following standard fact about Chebyshev polynomials (see, e.g.,~\cite{Cheney:66, KS01}).
    \begin{fact} \label{fct:chebyshev-derivative}
        Let $T_d:\R\to \R$ be the degree-$d$ Chebyshev polynomial.
        Then $T_d$ satisfies the following: 
        \begin{enumerate}
            \item $|T_d(t)\leq 1|$ for $|t|\leq 1$ with $T_d(1)=1$; and
            \item $T_d'(t)\ge d^2$ for $t>1$ with $T_d'(1)= d^2$. 
        \end{enumerate}      
    \end{fact}
    
\noindent For the case that $\eps\in [1/4,1/2)$, we construct the 
    approximate polynomial as follows.
    Firstly, we take the univariate polynomial 
    $p_1(t)= T_d\left (\frac{r-t}{r-1}\right )$, where $d=\lceil 2r^{1/2}(1-2\eps)^{1/2}\rceil$.
    Notice that $p_1(0)\ge 1+4(1-2\eps)$ and $p_1(t)\in [-1,1]$ 
    for any $t\in [1,r]$ from \Cref{fct:chebyshev-derivative}. 
    We then rescale $p_1$ and take the Hamming weight of $\bx$ 
    on $S$ as input. Namely, we define the multivariate polynomial 
    $p(\bx)=-(1-\eps)p_1(W_S(\bx))/p_1(0)+1$.
    Notice that for $W_S(\bx)=0$, $p(\bx)=\eps$.
    For $W_S(\bx)\in [1,r]$, we have 
    \[|(1-\eps)p_1(W_S(\bx))/p_1(0)|\le \frac{1-\eps}{1+4(1-2\eps)}\le \frac{\eps+(1-2\eps)}{1+4(1-2\eps)}\leq \max(\eps,1/4)=\eps\;.\]
    Therefore,  for $W_S(\bx)\in [1,r]$, we get 
    $p(\bx)\in [1-\eps,1+\eps]$.
    Furthermore, the degree of the polynomial $p$ is 
    $O(r^{1/2}(1-2\eps)^{1/2})$.
    
    For the case that $\eps\in (0,1/4)$, we construct the approximate polynomial $p$ similarly.
    Firstly, we take the univariate polynomial 
    $p_1(t)=T_d\left (\frac{r-t}{r-1}\right )/2$, 
    where $d=\lceil 2r^{1/2}\rceil$.
    Notice that $p_1(0)\ge 1$ and $p_1(t)\in [-1/2,1/2]$ 
    for any $t\in [1,r]$ from \Cref{fct:chebyshev-derivative}.
    Then we take $p_2(t)=p_1(t)^{c\log (1/\eps)}$, for a sufficiently 
    large constant $c$, which has $p_2(0)\ge 1$ and $p_2(t)\in [-\eps,\eps]$ for any $t\in [1,r]$.
    The multivariate polynomial defined as 
    $p(\bx)=-p_2\left (W_S(\bx)\right )/p_2(0)+1$ 
    has the desired property.
    The degree of the polynomial $p$ is $O(r^{1/2}\log 1/\eps)$.
    This completes the proof.
\end{proof}

Given the definitions of $X_\heavy$ and $X_\light$, 
since $\opt\leq 1/2-\eps$, the target disjunction $f_S$ 
is $\Omega(\eps)$-correlated with the labels 
on either $X_\light$ or $X_\heavy$. That is, we have that either 
$\E_{(\bx,y)\sim D}[f_S(\bx)(2y-1)\Ind(\bx\in X_\light)]=\Omega(\eps)$ or 
$\E_{(\bx,y)\sim D}[f_S(\bx)(2y-1)\Ind(\bx\in X_\heavy)]=\Omega(\eps)$. 
Our algorithm proceeds by one of the following cases:
\begin{enumerate}[leftmargin=*]
    \item \label{itm:sample-case-1}
    Suppose that $f_S$ is $\Omega(\eps)$-correlated 
    with the labels on $X_\light$.
    Then, since all the points in $X_\light$ have Hamming weight at most $r=n^{2/3}$, 
    the $\eps$-approximate degree of monotone disjunctions on $X_\light$ 
    is at most $O(r^{1/2}\log(1/\eps))=O(n^{1/3}\log(1/\eps))$ 
    by  \Cref{lem:approx-degree-coordinate}.
    Therefore, we can simply apply the standard $L_1$-polynomial 
    regression with degree-$O(n^{1/3}\log(1/\eps))$ to get a hypothesis with error $1/2-\Omega(\eps)$ (see \Cref{fct:l1-regression}).

    \item \label{itm:sample-case-2}
    Suppose that $f_S$ is $\Omega(\eps)$-correlated 
    with the labels on $X_\heavy$ and the labels are not balanced on $X_\heavy$, i.e., $|\pr_{(\bx,y)\sim D}[y=1\mid \bx\in X_\heavy]-1/2|\geq c\eps$,   
    for some constant $c>0$.  
    In this case, the constant classifier $h(\bx)\equiv 0$ or $h(\bx)\equiv 1$ works as a weak agnostic learner on $X_\heavy$.
    Then we can find some constant $c'\in \{0,1\}$ such that $\pr_{(\bx,y)\sim D}[y\neq c'\mid \bx\in X_{\light}]\leq 1/2$, and the hypothesis $h(\bx)= 1$ if $\bx\in X_\heavy$; $h(\bx)= c'$ otherwise, should suffice for our weak agnostic learner.
    
    \item \label{itm:sample-case-3}
    If neither of the above two cases hold,
    then we use rejection sampling to sample $\bx_{\guess}\sim D$ 
    conditioned on $\bx_{\guess}\in X_\heavy$.
    Since $f_S$ has $1/2-\Omega(\eps)$ error and the labels are approximately 
    balanced in $X_\heavy$, with probability $\Omega(\eps)$ 
    we have that $f_S(\bx_\guess)=0$. 
    Suppose that we correctly guess an $\bx_\guess$ such that $f_S(\bx_\guess)=0$.
    Then by removing any coordinate $i$ such that $(\bx_\guess)_i=1$, we remove at least $r=n^{2/3}$ coordinates from future consideration, since these coordinates cannot be in $S$. 
    Then the algorithm proceeds by recursing on the remaining coordinates.
\end{enumerate}
Notice that if neither \Cref{itm:sample-case-1} nor \Cref{itm:sample-case-2} hold, then \Cref{itm:sample-case-3} must hold; therefore, the algorithm can always make progress.
It is easy to see that the depth of the recursion is $O(n^{1/3})$ 
and the overall success probability 
is $2^{-\Tilde{O}(n^{1/3}\log (1/\eps))}$. 

\smallskip

A detailed pseudocode is given in \Cref{alg:weak-learner-non-sq}.
We are now ready to prove \Cref{thm:main-weak-non-sq}.

\begin{algorithm}
    \caption{Weak Agnostic Learning of Monotone Disjunctions}
    \label{alg:weak-learner-non-sq}
        \textbf{Input:} $\eps\in (0,1/2)$ and sample access to a  
        distribution $D$ of $(\bx,y)$ supported on 
        $\{0, 1\}^{n}\times \{0, 1\}$, such that there is a monotone disjunction $f_S(\bx)=\bigvee_{i\in S} \bx_i$ 
        with $\pr_{(\bx,y)\sim D}[f_S(\bx)\neq y]\leq 1/2-\eps$.
        
        \textbf{Output:} With $2^{-O(n^{1/3}\log(1/\eps))}$ probability, 
        a hypothesis $h:\{0, 1\}^{n}\to \{0, 1\}$ such that 
        $\pr_{(\bx,y)\sim D}[h(\bx)\neq y]\leq 1/2-\eps/100$.
        
        \vspace{0.3cm}

        \begin{algorithmic} [1]
            \State Set $r \leftarrow n^{2/3}$, $T \leftarrow \lceil n/r\rceil+1$ 
            and initialize $I_0\leftarrow [n]$.\label{line:non-SQ-initialization}
            \Statex $\triangleright$ $I$ keeps track of the remaining coordinates for consideration. 

            \State Let $P$ be a set of $2^{\tilde{O}(n^{1/3}\log(1/\eps))}$ i.i.d.\ samples from $D$ (with sufficiently large implied constant).
            
            \For{$t=\{0,\cdots, T\}$} \label{line:non-SQ-inner-loop}
                
                \State \algparbox{Define $X_\light\eqdef \{\bx\in\{0, 1\}^n \mid W_{I_t}(\bx)\leq r\}$, 
                $X_\heavy\eqdef \{\bx\in\{0, 1\}^n \mid W_{I_t}(\bx)> r\}$ and partition $P$ as $P_\light\eqdef \{(\bx,y)\in P \mid \bx\in X_\light\}$ 
                and $P_\heavy\eqdef \{(\bx,y)\in P \mid \bx\in X_\light\}$. }
                
                \State 
                \algparbox{ Apply $L_1$-regression on $u(P_\light)$ 
                for degree-$O\left (r^{1/2}\log (1/\eps)\right )$, which succeeds with at least constant probability}. Let $h_1'$ be the output hypothesis. \label{line:l1-regression}

                \State \algparbox{Sample $c'\sim u(\{0,1\})$ and define the hypothesis $h_1:\{0,1\}^n\to \{0,1\}$ as $h_1(\bx)=h_1'(\bx)$ if $\bx\in X_\light$; and $h_1(\bx)=c'$ otherwise.}

                \State \algparbox{Sample $c'\sim u(\{0,1\})$ and define the hypothesis $h_2:\{0,1\}^n\to \{0,1\}$ as $h_2(\bx)=1$ if $\bx\in X_\heavy$; and $h_2(\bx)=c'$ otherwise.}

                \State \algparbox{Let $\hat{\err}_i$ be
                $\pr_{(\bx,y)\sim u(P)}[h_i(\bx)\neq y]$ for $i\in \{1,2\}$.\label{line:non-sq-empirical-estimation}}                
                \If{$\hat{\err}_i\leq 1/2-\eps/10$ for any $i\in \{1,2\}$} \label{line:non-sq-success}
                \State {\bf return} the $h_i$ that satisfies the above condition.
                \EndIf
                \State \algparbox{
                Sample a random $(\bx_\guess,y)\sim u(P_\heavy)$.
                \label{line:non-SQ-sample-negative} } 
                \State Update $I_{t+1}\leftarrow I_t\backslash \{i\in [n]\mid (\bx_\guess)_i=1\}$.\label{line:non-SQ-remove-coordinates} 
                \EndFor
        \end{algorithmic}
\end{algorithm}

\begin{proof} [Proof of \Cref{thm:main-weak-non-sq}]    
    We first analyze the sample and computational complexity.
    Since the algorithm only uses samples in $P$, which contains $2^{\tilde{O}(n^{1/3}\log(1/\eps))}$ i.i.d.\ samples from $D$,
    the sample complexity of the algorithm is at most $2^{\tilde{O}(n^{1/3}\log(1/\eps))}$.
    Furthermore, the computational complexity of the algorithm is at most $T\poly\left (|P|,d^{O\left (r^{1/2}\log (1/\eps)\right )}\right )=2^{\tilde{O}(n^{1/3}\log(1/\eps))}$.

    We then prove the correctness of our algorithm. 
    We first show that the algorithm, with probability at least 
    $2^{-O(n^{1/3}\log(1/\eps))}$ returns a hypothesis $h:\{0,1\}^n\to \{0,1\}$ such that 
    ${\pr_{(\bx,y)\sim u(P)}[h(\bx)\neq y]}\\\leq 1/2-\eps/10$ over the sample set $P$.
    Since we only need to show that the algorithm succeeds with $2^{-O(n^{1/3}\log(1/\eps))}$ probability,
    we assume that for all the $h_1$ and $h_2$, the algorithm always chooses the $c'\in \{0,1\}$ that minimize 
    $\pr_{(\bx,y)\sim u(P)}[h_1(\bx)\neq y]$ and $\pr_{(\bx,y)\sim u(P)}[h_2(\bx)\neq y]$.
    Given the algorithm only runs for $T$ iterations, this happens with probability at least 
    $2^{-O(T)}=2^{-O(n^{1/3})}$.     
    Let $f_S(\bx)=\bigvee_{i\in S} \bx_i$ 
    be an arbitrary optimal monotone disjunction on $D$.
    From the assumption of the algorithm,
    we have that 
    $\pr_{(\bx,y)\sim D}[f_S(\bx)\neq y]\leq 1/2-\eps$.
    Then using \Cref{fct:vc-inequality} and the fact that the VC-dimension of the disjunctions is $n$, we have that
    $\pr_{(\bx,y)\sim u(P)}[f_S(\bx)\neq y]\leq 1/2-\eps/2$.
    We first show the following structural lemma.
    \begin{lemma} \label{lem:non-SQ-helper}
        Suppose that $\pr_{(\bx,y)\sim u(P)}[f_S(\bx)\neq y]\leq 1/2-\eps/2$, then for any iteration of \Cref{alg:weak-learner-non-sq}, at least one of the following holds:
        \begin{enumerate}
            \item $\pr_{(\bx,y)\sim u(P)}[y=1\land \bx\in X_\heavy]-\pr_{(\bx,y)\sim u(P)}[y=0\land \bx\in X_\heavy]\geq \eps/4$, \label{prop:constant-inbalance}
            \item $\pr_{(\bx,y)\sim u(P)}[f_S(\bx)=y\land \bx\in X_\light]-\pr_{(\bx,y)\sim u(P)}[f_S(\bx)\neq y\land \bx\in X_\light]\ge \eps/4$, or \label{prop:adv-light}
            \item $\pr_{(\bx,y)\sim u(P)}[f_{S}(\bx)=0\land \bx\in X_\heavy]\ge \eps/4$.
        \end{enumerate}
    \end{lemma}
    \begin{proof}
        Notice that $\pr_{(\bx,y)\sim u(P)}[f_S(\bx)\neq y]\leq 1/2-\eps/2$ implies that 
        \[ {\pr_{(\bx,y)\sim u(P)}[f_S(\bx)= y]}-{\pr_{(\bx,y)\sim u(P)}[f_S(\bx)\neq y]}\ge \eps\; . \]
        Suppose that neither \Cref{prop:constant-inbalance} nor \Cref{prop:adv-light} hold.
        Since \Cref{prop:adv-light} does not hold and the above inequality, we get
        \[\pr_{(\bx,y)\sim u(P)}[f_S(\bx)=y\land \bx\in X_\heavy]-\pr_{(\bx,y)\sim u(P)}[f_S(\bx)\neq y\land \bx\in X_\heavy]\ge (3/4)\eps\; .\]        
        Combining this and the fact that \Cref{prop:constant-inbalance} does not hold, we get 
        \begin{align*}
            &\pr_{(\bx,y)\sim u(P)}[f_S(\bx)=0\land \bx\in X_\heavy]\ge\pr_{(\bx,y)\sim u(P)}[f_S(\bx)=y\land y=0\land \bx\in X_\heavy]\\
            \ge &\pr_{(\bx,y)\sim u(P)}[f_S(\bx)=y\land \bx\in X_\heavy]+\pr_{(\bx,y)\sim u(P)}[y=0\land \bx\in X_\heavy]-\pr_{(\bx,y)\sim u(P)}[\bx\in X_\heavy]\\
            \ge &\frac{1}{2}\left(\pr_{(\bx,y)\sim u(P)}[\bx\in X_\heavy]+(3/4)\eps\right)
            +\frac{1}{2}\left(\pr_{(\bx,y)\sim u(P)}[\bx\in X_\heavy]-\eps/4 \right)
            -\pr_{(\bx,y)\sim u(P)}[\bx\in X_\heavy]\\
            = &\eps/4\;. 
        \end{align*}        
    \end{proof}
    Notice that if the algorithm terminates before $T$ iterations and returns a hypothesis, it must return an $h$ such that ${\pr_{(\bx,y)\sim u(P)}[h(\bx)\neq y]}\leq 1/2-\eps/10$.
    Furthermore, suppose $P_\heavy$ is always nonempty except in the iteration it terminates. 
    Then Line~\ref{line:non-SQ-remove-coordinates} must remove at least $r$ coordinates from $S$ in each iteration (from the definition of $P_\heavy$), and the algorithm must terminate in $\lceil n/r\rceil\leq T$ iterations.
    Therefore, it suffices for us to show that $P_\heavy$ is always nonempty except in the iteration it terminates.
    
    We first argue that, with 
    at least $2^{-O(n^{1/3}\log(1/\eps))}$ probability, 
    $S\subseteq I_t$ for all iterations.
        For iteration $t$, suppose that $S\subseteq I_t$ and the algorithm did not terminate in iteration $t$.
        Then the ``if'' condition in Line~\ref{line:non-sq-success} is not satisfied for both $h_1$ and $h_2$.
        Since $\pr_{(\bx,y)\sim u(P)}[h_1(\bx)\neq y]> 1/2-\eps/10$,
        given that the algorithm always chooses the $c'$ that minimizes $\pr_{(\bx,y)\sim u(P)}[h_1(\bx)\neq y]$, we have that
        \[\pr_{(\bx,y)\sim u(P)}[h_1'(\bx)=y \land \bx\in X_\light]-
        \pr_{(\bx,y)\sim u(P)}[h_1'(\bx)\neq y \land \bx\in X_\light]\leq \eps/5\; .\]
        From \Cref{fct:l1-regression}, \Cref{lem:approx-degree-coordinate} and $S\subseteq I_t$,
        we have that $L_1$ regression learns a hypothesis $h_1'$ such that 
        \[\pr_{(\bx,y)\sim u(P_\light)}[h_1'(\bx)\neq y]\leq \pr_{(\bx,y)\sim u(P_\light)}[f_S(\bx)\neq y]+\eps/100\; .\]
        This implies that 
        \[\pr_{(\bx,y)\sim u(P)}[f_S(\bx)=y \land \bx\in X_\light]-
        \pr_{(\bx,y)\sim u(P)}[f_S(\bx)\neq y \land \bx\in X_\light]< \eps/4\; ,\]
        and therefore \Cref{prop:adv-light} in \Cref{lem:non-SQ-helper} does not hold.
        Then similarly, since 
        $\pr_{(\bx,y)\sim u(P)}[h_2(\bx)\neq y]> 1/2-\eps/10$, 
        it follows that 
        \[\pr_{(\bx,y)\sim u(P)}[y=1 \land \bx\in X_\heavy]-
        \pr_{(\bx,y)\sim u(P)}[y=0 \land \bx\in X_\heavy]< \eps/4\; ,\]
        and therefore \Cref{prop:constant-inbalance} in \Cref{lem:non-SQ-helper} does not hold.
        Combining the above arguments and \Cref{lem:non-SQ-helper}, we get that 
        $\pr_{(\bx,y)\sim u(P)}[f_{S}(\bx)=0\land \bx\in X_\heavy]\ge \eps/4$ (this also implies that $P_\heavy$ is nonempty).
        Therefore, with probability at least $\Omega(\eps)$, the algorithm always samples an $\bx_\guess$ such that $f_S(\bx_\guess)=0$ in Line~\ref{line:non-SQ-remove-coordinates}, and any coordinate $i$ removed from $I_t$ in Line~\ref{line:non-SQ-remove-coordinates}
        cannot be in $S$.
        Therefore, we have $S\subseteq I_{t+1}$ with probability at least $\Omega(\eps)$.

        Suppose the algorithm runs for $T'$ iterations.
        We argue that with 
        at least $2^{-O(n^{1/3}\log(1/\eps))}$ probability, 
        $S\subseteq I_t$ for all iteration $t\in [T']$.
        Notice that we initialized $I_0=[n]$, 
        which satisfies $S\subseteq I_0$.
        For iteration $t$, given that $S\subseteq I_t$, 
        if the algorithm does not terminate, 
        then with probability at least $\Omega(\eps)$,
        $S\subseteq I_{t+1}$ (as argued in the paragraph above).
        Since there are at most $T$ iterations, 
        with probability at least 
        $\Omega(\eps)^T=2^{-O(n^{1/3}\log(1/\eps))}$, 
        we will have $S\subseteq I_t$ for all iterations $t\in [T']$.
        This also implies that $P_\heavy$ is always nonempty 
        except in the last iteration (as we argued above).
        Therefore, the algorithm must terminate in at most $T$ iterations 
        and return a hypothesis $h$ such that 
        $\pr_{(\bx,y)\sim u(P)}[h(\bx)\neq y]\leq 1/2-\eps/10$ 
        over the sample set $P$.
 
     It only remains for us to show that any hypothesis $h$ that the algorithm returns such that $\pr_{(\bx,y)\sim u(P)}[h(\bx)\neq y]\leq 1/2-\eps/10$ must also satisfies $\pr_{(\bx,y)\sim D}[h(\bx)\neq y]\leq 1/2-\eps/100$.
    To do so, notice that both $h_1$ and $h_2$ are always 
    an intersection of two halfspaces. Therefore, the VC-dimension 
    of the class of functions of all possible hypotheses 
    the algorithm can return is $O(n)$.
    By \Cref{fct:vc-inequality}, we have that 
    \[\pr_{(\bx,y)\sim D}[h(\bx)\neq y]\leq \pr_{(\bx,y)\sim u(P)}[h(\bx)\neq y]+\eps/100\leq 1/2-\eps/100\; .\]
    This completes the proof.
    \end{proof}

\section{Statistical Query Agnostic Learner} \label{sec:sq-ub}
Here we provide a Statistical Query 
version of the previous algorithm with qualitatively the same complexity. 
Interestingly, this algorithm can be implemented in the SQ model, 
but not in the CSQ model.
Furthermore, it outperforms the CSQ lower bound of~\cite{GKK20} 
(see Appendix~\ref{sec:csq-lb}), 
which implies a strong separation between the power of SQ 
and CSQ algorithms in the distribution-free agnostic model.
For concreteness, we include the basics on the SQ model.

\paragraph{Statistical Query (SQ) Model}
The class of SQ algorithms~\cite{Kearns:98} is a family 
of algorithms that are allowed
to query expectations of bounded functions on the underlying distribution 
through an (SQ) oracle
rather than directly access
samples.

\begin{definition}[SQ Model] \label{def:sq}
Let $D$ be a distribution on $X\times \{0,1\}$. 
A \emph{statistical query} is a bounded function 
$q:X\times \{0,1\}\rightarrow[-1,1]$,
and we define $\stat_D(\tau)$ to be the oracle 
that given any such query $q$, outputs a value $v\in [-1,1]$ 
such that $|v-\E_{(\bx,y)\sim D}[q(\bx,y)]|\leq\tau$,
where $\tau>0$ is the \emph{tolerance} parameter of the query.
An \emph{SQ algorithm} is an algorithm 
whose objective is to learn some information about an unknown 
distribution $D$ by making adaptive calls to the $\stat_D(\tau)$ 
oracle.
\end{definition}

\begin{theorem} \label{thm:main-strong-sq}
    Let $D$ be an unknown joint distribution of $(\bx,y)$ supported on $\{0, 1\}^{n}\times \{0, 1\}$ and $\eps\in (0,1)$. 
    There is an algorithm that makes at most $q$ queries to $\stat_D(\tau)$, has $T$ computational complexity, and distribution-free agnostically learns disjunctions to additive error $\eps$, where $\max(q,1/\tau,T)=2^{\Tilde{O}(n^{1/3}\log(1/\eps))}$.
\end{theorem}

It is worth noting that \Cref{thm:main-strong-sq} morally corresponds to a sample-based algorithm that has $2^{\Tilde{O}(n^{1/3}\log(1/\eps))}$ sample and computational complexity, as one can simulate the answers from the SQ oracle by empirical estimation of queries using fresh samples. 

The high-level intuition of the algorithm is the following.
As we have discussed in the previous section, it suffices to consider learning monotone disjunctions.
Let $f_S(\bx)=\bigvee_{i \in S} \bx_i$ be any optimal monotone disjunction.
We set the radius parameter $r=n^{2/3}$ and call a coordinate $i\in [n]$ heavy if $i\in S$ and $\E_{(\bx,y)\sim D}[\Ind(\bx_i=1)]\ge r/n$.
Given an input distribution $D$ of $(\bx,y)$, 
we will either (a) guess that some $i\in [n]$ is a heavy coordinate; 
or (b) guess that there is no heavy coordinate.
Suppose that with probability $1/2$ we sample $i\sim u([n])$
and guess that $i$ is heavy, and with probability $1/2$ 
we guess that there is no heavy coordinate. 
Then this guess is always correct with probability $\Omega(1/n)$. 
Suppose that our guess is correct.
\begin{enumerate}[leftmargin=*]
    \item Suppose we guessed that $i$ is heavy.
    Let $B=\{\bx\in\{0, 1\}^n\mid \bx_i=1\}$. Then we know that $f_S(\bx)=1$ for any $\bx\in B$, and we can remove any $\bx\in B$ 
    from the input distribution for further consideration.
    By the definition of heavy coordinates, the probability mass removed is $\pr_{(\bx,y)\sim D}[\bx\in B]\ge r/n$.
    
    \item Suppose we guessed that there is no heavy coordinate. 
    Then we can remove any coordinate $i$ such that 
    $\E_{(\bx,y)\sim D}[\Ind(\bx_i=1)]\ge r/n$, 
    since such a coordinate cannot be in $S$.
    Let $I$ be the set of remaining coordinates.
    Then the expected Hamming weight on $I$ satisfies 
    $\E_{(\bx,y)\sim D}[W_I(\bx)]\leq r$.
    Let $B=\{\bx\in\{0, 1\}^n\mid W_I(\bx)\leq 2r\}$.
    If we apply the degree-$O(r^{1/2}\log(1/\eps))$ 
    $L_1$-regression algorithm on $B$, 
    we will get a hypothesis $h$ such that the error of 
    $h$ on $B$ is at least as good as $f_S$ on $B$.
    Therefore, we can just label every $\bx\in B$ by $h(\bx)$ 
    and remove $B$ from the distribution for further consideration.
    Since $\E_{(\bx,y)\sim D}[W_I(\bx)]\leq r$ and 
    $B=\{\bx\in\{0, 1\}^n\mid W_I(\bx)\leq 2r\}$, 
    by Markov's inequality, 
    the mass removed is $\pr_{(\bx,y)\sim D}[\bx\in B]\ge 1/2$.
\end{enumerate}
Now let $D'$ be the new distribution of $(\bx,y)\sim D$ 
conditioned on $\bx\not \in B$ with the irrelevant coordinates removed.
We can repeat the above process on $D'$. 
Since each time we remove at least $r/n=n^{-1/3}$ fraction of the input distribution, we only need to guess correctly $n^{1/3}$ times.
In the end, our output hypothesis will be a decision list 
combining the partial classifiers on all the sets $B$ we removed.

The algorithm establishing \Cref{thm:main-strong-sq} is provided as \Cref{alg:learner-sq}.

\begin{algorithm}
    \caption{Distribution-free Agnostic Learning Monotone Disjunctions to Additive Error (SQ)}
    \label{alg:learner-sq}
        \textbf{Input:} $\eps\in (0,1/2)$ and SQ query access to a joint distribution $D$ of $(\bx,y)$ supported on $\{0, 1\}^{n}\times \{0, 1\}$.
        
        \textbf{Output:} With at least $2^{-\Tilde{O}(n^{1/3}\log(1/\eps))}$ probability, the algorithm outputs a hypothesis $h:\{0, 1\}^{n}\to \{0, 1\}$ such that $\pr_{(\bx,y)\sim D}[h(\bx)\neq y]\leq \opt+\eps$, where $\opt$ is the error of the optimal monotone disjunction.

        \begin{algorithmic} [1]
            \Statex $\triangleright$ For convenience of the algorithm description and analysis, we fix any 
            optimal monotone disjunction $f_S(\bx)=\bigvee_{i\in S} \bx_i$ 
            (unknown to the algorithm).
            \State Let $r=n^{2/3}$, $c$ be a sufficiently large constant, $T=cn\log (1/\eps)/r$ and initialize 
            $U_0\leftarrow \{0, 1\}^{n}$ and $I_0\leftarrow [n]$. 
            \Statex $\triangleright$ $U$ and $I$ keep track of the remaining domain and coordinates.
            \For{$t=\{0,\cdots, T\}$} 
                \State Define a coordinate $i$ as heavy if $i\in S$ and $\pr_{(\bx,y)\sim D} [\Ind (\bx_i=1)|\bx\in U_t]\ge r/n$.
                \State \algparbox{With probability $1/2$, 
                sample an $i\sim u(I_t)$, guess that $i$ is heavy and run \Call{RemoveHeavyCoordinate}{}.
                With the remaining $1/2$ probability, guess that there is no heavy coordinate and run \Call{$L_1$-RegressionOnLight}{}.} \label{line:SQ-guess}
                \State $U_{t+1}\leftarrow U_t\backslash B_t$.
                \State Let $\hat{P}_{U_{t+1}}$ be the answer of $\stat_D(\eps/100)$ for the query function $q_i(\bx,y)=\Ind(\bx\in U_t)$.
                \If{$\hat{P}_{U_{t+1}}\leq \eps/3$} \label{line:SQ-composite-hypothesis}
            \State \algparbox{
            Sample $c'\sim u(\{0,1\})$ and define $h:\{0, 1\}\to \{0, 1\}$ as 
            $h(\bx)=h'_k(\bx)$, where $k\in N$ is the smallest integer such that $\bx\in B_k$, and 
            $h(\bx)=c'$ otherwise. Return $h$.}
            \EndIf
        \EndFor    
        \end{algorithmic}        
        \hrulefill
        \begin{algorithmic} [1]
            \Procedure{RemoveHeavyCoordinate}{} 
            \State Let $B_t=\{\bx\in U_t|\bx_i=1\}$.
            \State \algparbox{Let $h'_t:B\to \{0, 1\}$ be a partial classifier on $B_t$ defined as $h'_t(\bx)=1$ if $\bx\in B_t$.}
            \State $I_{t+1}\leftarrow I_t\backslash \{i\}$.
            \EndProcedure
            \vspace{3mm}
        \end{algorithmic}
        \hrulefill           
        \begin{algorithmic} [1]
            \Procedure{$L_1$-RegressionOnLight}{} 
            \State \algparbox{Let $\hat{P}_U$ and $\hat{P}_i$ (for all $i\in I_t$) be the answer of $\stat_D(\eps r/(800n))$ for the query function $q_U(\bx,y)=\Ind(\bx\in U_t)$ and $q_i(\bx,y)=\Ind(\bx_i=1\land \bx\in U_t)$ respectively. \label{line:SQ-estimate-conditional} \\ \Comment{$\hat{P}_i/\hat{P}_U$ will be the empirical estimation of $\pr_{(\bx,y)\sim D}[\bx_i=1\mid \bx\in U_t]$.}}
            \State \algparbox{$I_{t+1}\leftarrow I_t\backslash \{i\mid \hat{P}_i/\hat{P}_U\ge (1+1/100)r/n\}$ and $B_t=\{\bx\in U_t|W_{I_{t+1}}(\bx)\leq 2r\}$.}
            \State \algparbox{Let $D'$ be the joint distribution of $(\bx,y)\sim D$ conditioned on $\bx\in B_t$, which we have SQ query access to by asking queries on the distribution $D$.}
            \State \algparbox{Apply the degree-$\left (cr^{1/2}\alpha^{-1/2}\right )$ polynomial $L_1$-regression algorithm (\Cref{fct:l1-regression}) on $D'$ and let $h'_t$ be the output hypothesis.}
            \EndProcedure
            \vspace{3mm}
        \end{algorithmic}
\end{algorithm}

We are now ready to prove the main theorem of this section.
\begin{proof} [Proof of \Cref{thm:main-strong-sq}]
For convenience of the analysis, let $f_S$ be the same optimal hypothesis we fixed in the algorithm to maintain a consistent definition of heavy coordinates.
We first prove the correctness of \Cref{alg:learner-sq}, i.e., the algorithm will, with probability at least 
$2^{-\Tilde{O}(n^{1/3}\log(1/\eps))}$, outputs a hypothesis $h$ such that $\pr_{(\bx,y)\sim D}[h(\bx)\neq y]\leq \pr_{(\bx,y)\sim D}[f_S(\bx)\neq y]+\eps$.
Notice that if there is any heavy coordinate in $I$, the algorithm will guess such a heavy coordinate with probability at least $1/(2n)$.
If there is no heavy coordinate, with probability $1/2$, the algorithm will guess that there is no heavy coordinate.
Therefore, the guess made by the algorithm is always correct with probability at least $1/(2n)$.
Suppose the algorithm runs for $T'$ iterations.
Then, with probability at least $\Omega(1/n)^{T'}\geq \Omega(1/n)^T=2^{-\Tilde{O}(n^{1/3}\log(1/\eps))}$, all the guesses made by the algorithm are correct,
and it suffices for us to show that 
if all the guesses are correct, then the algorithm outputs a hypothesis with error at most $\opt+\eps$ deterministically.
We first prove the following lemma on the partial classifiers 
$h_t$ obtained in each iteration.
\begin{lemma} \label{lem:sq-algorithm-error-progress}
    In \Cref{alg:learner-sq}, given that all the guesses are correct,
    then for both \textsc{RemoveHeavyCoordinate} and \textsc{$L_1$-RegressionOnLight} procedures, we have the following properties:
    \begin{enumerate}
        \item \label{property:small-error}
        $\pr_{(\bx,y)\sim D} [h_t(\bx)\neq y\land \bx\in B_t]-\pr_{(\bx,y)\sim D} [f_S(\bx)\neq y\land \bx\in B_t]\leq \eps/(100T)$ ; and

        \item \label{property:large-mass}
        $\pr_{(\bx,y)\sim D}[\bx\in B_t|\bx\in U_t]\ge n^{-1/3}$.
    \end{enumerate}
\end{lemma}
\begin{proof}
    We just need to show that the two properties hold for the output hypothesis of both \textsc{RemoveHeavyCoordinate} procedure and \textsc{$L_1$-RegressionOnLight} procedure.
    It is easy to see that for the procedure \textsc{RemoveHeavyCoordinate}, both properties follow from the definition of heavy coordinate.
    So, it only remains to verify them 
    for the \textsc{$L_1$-RegressionOnLight} procedure.
    Notice that in the \textsc{$L_1$-RegressionOnLight} procedure, 
    we want to remove any coordinate $i$ such that $\pr_{(\bx,y)\sim D}[\bx_i=1\mid \bx\in U_t]\geq r/n$, since they cannot be in $S$. 
    We do so by using $\hat{P}_i/\hat{P}_U$ as an estimate of $\pr_{(\bx,y)\sim D}[\bx_i=1\mid \bx\in U_t]$.
    We will first need the following fact, which shows that $\hat{P}_i/\hat{P}_U$ is an accurate estimate.
    \begin{fact} \label{fct:SQ-oracle-conditional}
        Let $P_1,P_2,\hat{P_1},\hat{P}_2\in [0,1]$ with $P_1\leq P_2$, $|P_1-\hat{P}_1|\leq \tau$, $|P_2-\hat{P}_2|\leq \tau$ and  $\hat{P}_2-\tau\geq \gamma>0$,.
        Then we have $|\hat{P}_1/\hat{P}_2-P_1/P_2|\leq 2\tau/\gamma$.
    \end{fact}
    \begin{proof}
        For the direction $\hat{P}_1/\hat{P}_2-P_1/P_2\ge 2\tau/\gamma$, we have
    \begin{align*}
    P_1/P_2
    \le \frac{\hat{P}_1+\tau}{\hat{P}_2-\tau}
    = \left (\frac{\hat{P}_1}{\hat{P}_2}+\frac{\tau}{\hat{P}_2}\right )\frac{\hat{P}_2}{\hat{P}_2-\tau}
    \le &\frac{\hat{P}_1}{\hat{P}_2}+2\frac{\tau}{\hat{P}_2-\tau}=\frac{\hat{P}_1}{\hat{P}_2}+2\tau/\gamma\; .
    \end{align*}
    For the direction $\hat{P}_1/\hat{P}_2-P_1/P_2\le 2\tau/\gamma$, we have
    \begin{align*}
    P_1/P_2
    \ge \frac{\hat{P}_1-\tau}{\hat{P}_2+\tau}
    = \left (\frac{\hat{P}_1}{\hat{P}_2}-\frac{\tau}{\hat{P}_2}\right )\frac{\hat{P}_2}{\hat{P}_2+\tau}
    \ge &\frac{\hat{P}_1}{\hat{P}_2}-2\frac{\tau}{\hat{P}_2+\tau}=\frac{\hat{P}_1}{\hat{P}_2}-2\tau/\gamma\; .
    \end{align*}
    \end{proof}
    A direct application of \Cref{fct:SQ-oracle-conditional} implies that    
    \begin{equation} \label{eq:SQ-oracle-conditional}
    \left |\hat{P}_i/\hat{P}_U-\pr_{(\bx,y)\sim D}[\bx_i=1\mid \bx\in U_t]\right |\leq (1/100)r/n \; .
    \end{equation}
    For Property~\ref{property:large-mass}, since we have removed any $i\in I_t$ such that $\hat{P}_i/\hat{P}_U\ge (1+1/100)r/n$, it follows from \Cref{eq:SQ-oracle-conditional} that for any remaining $i\in I_{t+1}$, $\pr_{(\bx,y)\sim D}[\bx_i=1\mid \bx\in U_t]\le (1+1/50)r/n$.
    Therefore, 
    $\E_{(\bx,y)\sim D}[W_I(\bx)\mid \bx\in U_t]=\sum_{i\in I} \E_{(\bx,y)\sim D}[\Ind(\bx_i=1)\mid \bx\in U_t]\leq (4/3)r/n$. By Markov's inequality, we have $\pr[\bx\in B_t\mid \bx\in U_t]\ge 1/3$.

For Property \ref{property:small-error}, we first show that $f_{S\cap I_{t+1}}(\bx)=f_{S}(\bx)$ for any $\bx\in U_t$, where $f_{S\cap I_{t+1}}(\bx)=\bigvee_{i\in S\cap I_{t+1}} \bx_i$.
Notice that any coordinate $i$ removed from $I_t$ must have $\hat{P}_i/\hat{P}_U\ge n^{-1/3}+n^{-1/3}/100$.
Therefore, by \eqref{eq:SQ-oracle-conditional}, any such coordinate $i$ must satisfy $\pr_{(\bx,y)\sim D}[\bx_i=1\mid \bx\in U]\geq n^{-1/3}$.
Given that all the guesses are correct, 
any such coordinate $i$ cannot be in $S$.
Therefore, we have that any $i\in S\backslash I_{t+1}$ must be removed from $I$ by the previous call to the 
\textsc{RemoveHeavyCoordinate} procedure.
Since the \textsc{RemoveHeavyCoordinate} procedure 
removed any $\bx$ such that $\bx_i=1$ from $U$, 
we must have that for any 
$i\in S\backslash I_{t+1}$ and $\bx\in U_t$, $\bx_i=0$.
Therefore, for any $\bx\in U_t$, 
\[f_{S\cap I_{t+1}}(\bx)=\bigvee_{i\in S\cap I_{t+1}} \bx_i=\bigvee_{i\in S\cap I_{t+1}} \bx_i\lor \bigvee_{i\in S\backslash I_{t+1}} \bx_i=\bigvee_{i\in S} \bx_i=f_{S}(\bx)\; .\] 
Then from \Cref{lem:approx-degree-coordinate}, we have the $\eps/(100T)$-approximate degree of disjunctions on $B_t$ is 
$$O(n^{1/3}\log(T/\eps))=\tilde{O}(n^{1/3}\log(1/\eps)) \;.$$
Using \Cref{fct:l1-regression}, we get that 
\[\pr_{(\bx,y)\sim D} [h_t(\bx)\neq y\mid \bx\in B_t]-\pr_{(\bx,y)\sim D} [f_{S\cap I_{t+1}}(\bx)\neq y\mid \bx\in B_t]\leq \eps/(100T)\; .\]
Since $f_{S\cap I_{t+1}}(\bx)=f_{S}(\bx)$ for all $\bx\in U_t$, we have
\[\pr_{(\bx,y)\sim D} [h_t(\bx)\neq y\mid \bx\in B_t]-\pr_{(\bx,y)\sim D} [f_{S}(\bx)\neq y\mid \bx\in B_t]\leq \eps/(100T)\;,\]
which implies 
\[\pr_{(\bx,y)\sim D} [h_t(\bx)\neq y\land \bx\in B_t]-\pr_{(\bx,y)\sim D} [f_S(\bx)\neq y\land \bx\in B_t]\leq \eps/(100T)\; .\]
This completes the proof.
\end{proof}
Given Property~\ref{property:large-mass} of Fact \ref{lem:sq-algorithm-error-progress}, since we are removing $B_t$ from $U_t$ in each iteration, we have that $\pr_{(\bx,y)\sim D}[\bx\in U_t]$ will decrease by a multiplicative factor of $1-n^{-1/3}$ in each iteration.
Therefore, after at most $T$ iterations, 
we have $\pr_{(\bx,y)\sim D}[\bx\in U_t]\leq \eps/4$ 
and the algorithm returns the hypothesis $h$ 
and terminates via Line~\ref{line:SQ-composite-hypothesis}. 
Suppose the algorithm terminates at the $T'$th iteration. 
Then the error of the output hypothesis is
\vspace{-0.2cm}
\begin{align*}
\pr_{(\bx,y)\sim D}[h(\bx)\neq y]\leq &\sum_{t=1}^{T'}\pr_{(\bx,y)\sim D}[h_t(\bx)\neq y\land \bx\in B_i]+\pr_{(\bx,y)\sim D}[\bx\in U_{T'+1}]\\
\leq & \pr_{(\bx,y)\sim D}[f_S(\bx)\neq y]+T'(\eps/(100T))+\eps/2\\
\leq  &\opt+\eps\; .
\end{align*}

It only remains to verify the query and computational complexity of \Cref{alg:learner-sq}. Notice that the smallest tolerance of any query that the algorithm directly asks is at least $\eps r/(800n)=\eps n^{-1/3}/800$,
and any query asked by the $L_1$-regression has tolerance at least $2^{-\Tilde{O}(n^{1/3}\log(1/\eps))}$.
Furthermore, the computational complexity of the algorithm is $Tn^{O(n^{1/3}\log(T/\eps))}=2^{\Tilde{O}(n^{1/3}\log(1/\eps))}$ and the total number of queries the algorithm asks must be bounded by the same quantity.
This completes the proof of the correctness of \Cref{alg:learner-sq}.

Given that \Cref{alg:learner-sq} is correct, 
we can simply repeat \Cref{alg:learner-sq} for 
$N=2^{\tilde{O}(n^{1/3}\log(1/\eps))}$ times 
with the error parameter set to $\eps/3$, 
and let $h_1,\cdots,h_N$ be the output hypotheses.
Let $\hat{\err}_i$ be the answer of $\stat_D(\eps/3 )$ for the query function $q(\bx,y)=\Ind(h_i(\bx)\neq y)$, 
which estimates the error of each output hypothesis.
Then simply output $h_k$, where $k={\rm argmin}_k \hat{\err}_i$.
The analysis here is straightforward.
Since \Cref{alg:learner-sq} succeeds with probability at least $2^{-\Tilde{O}(n^{1/3}\log(1/\eps))}$ and we repeat it for 
$N$ times, with probability at least a constant, 
the algorithm succeeds at least once. 
Therefore, we must have $\hat{\err}_k\leq \opt+2\eps/3$.
This implies that $\pr_{(\bx,y)\sim D}[h_k(\bx)\neq y]\leq \opt+\eps$ 
from the definition of the SQ oracle.
This gives us an SQ algorithm for distribution-free agnostic 
learning monotone disjunctions to additive error $\eps$.
To learn general disjunctions,  
as we have discussed in the previous section,
one can easily reduce learning general disjunctions 
to learning monotone disjunctions 
by including negated variables as additional features. 
This completes the proof.
\end{proof}

\section{Agnostic Learning with Approximate Error Guarantees} \label{sec:tradeoff-alg}

In the setup of Definition~\ref{def:agnostic-learning}, 
given an approximation factor $\alpha\in (1,\infty)$ 
and additive error $\eps\in(0,1)$, if an 
algorithm $\A$ outputs a hypothesis $h:X\to \{0, 1\}$ 
such that $\pr_{(\bx,y)\sim D}[h(\bx)\neq y]\leq \alpha \, \opt+\eps$, 
we will say that $\A$ $(\alpha,\eps)$-approximately agnostically learns $\C$. 
In this section, we give an SQ algorithm that provides 
a smooth trade-off between error and complexity. 
In particular, assuming $\alpha\in [32,\sqrt{n}]$ and $\eps\in (0,1)$,
there is an algorithm 
$\A$ that asks $q$ queries to $\stat_D(\tau)$, has computational complexity $T$,
and 
$(\alpha,\eps)$-approximately agnostically learns disjunctions,
where $\max(q,1/\tau,T)=2^{\Tilde{O}(n^{1/3}\alpha^{-2/3})}\poly(1/\eps)$.
Therefore, on one extreme point of the trade-off curve, 
we recover the guarantee of \Cref{alg:learner-sq} 
from \Cref{sec:sq-ub} (the requirement $\alpha\ge 32$ here 
is only used for convenience of the algorithm description); 
and on the other extreme point we recover the guarantee of an earlier 
algorithm by \cite{Pel07} (see also \cite{ABS10})
that runs in polynomial time 
and outputs a hypothesis with error $O(n^{1/2})\opt+\eps$.

We give the main theorem of this section.
\begin{theorem}[Approximate Agnostic Learner] \label{thm:main-tradeoff}
    Let $D$ be an unknwon joint distribution of $(\bx,y)$ supported on $\{0,1\}^n\times \{0,1\}$, $\alpha\in [32,\sqrt{n}]$ and $\eps\in (0,1)$.
    Then there is an algorithm that makes at most $q$ queries to $\stat_D(\tau)$, has $T$ computational complexity, 
    and $(\alpha,\eps)$-approximately agnostically learns disjunctions,
    where $\max(q,1/\tau,T)=2^{\Tilde{O}(n^{1/3}\alpha^{-2/3})}\poly(1/\eps)$. 
\end{theorem}

Similar to how we described the high-level intuition 
of the algorithm in \Cref{sec:non-sq-ub},
we start by pointing out two simplifications 
that can be made without loss of generality. 
First, it suffices to consider {\em monotone} disjunctions, 
as discussed in the previous sections.
Second, similar to \Cref{sec:non-sq-ub}, 
it suffices to develop a {\em weak} agnostic learner 
with the desired complexity.
In the context of $(\alpha,\eps)$-approximate agnostic learning, 
a corresponding weak learner is an algorithm 
whose hypothesis performs slightly better 
than a random guess, when the input distribution satisfies $\opt\le 1/(2\alpha)$. 
Given such an algorithm, we can leverage standard agnostic 
boosting techniques to obtain our 
$(\alpha,\eps)$-approximate agnostic learner, 
i.e., an algorithm with accuracy $\alpha\opt+\eps$, 
with qualitatively the same complexity 
(up to a polynomial factor). 

Specifically, it suffices to establish the following result:
\begin{theorem} [Weak Learner for Monotone Disjunctions given $\opt\le 1/\alpha$] \label{thm:main-tradeoff-weak}
    Let $D$ be an unknown distribution supported on 
    $\{0, 1\}^{n}\times \{0, 1\}$, $\alpha\in [64,\sqrt{n}]$ and $\eps\in (0,1)$. 
    Suppose there is a monotone disjunction 
    $f: \{0, 1\}^{n}\to \{0, 1\}$ such that 
    $\pr_{(\bx,y)\sim D}[f(\bx)\neq y]\leq\ 1/\alpha$.
    Then there is an algorithm that makes at most $q$ queries to $\stat_D(\tau)$, has $T$ computational complexity, and with probability at least $p$
    returns a hypothesis $h:\{0, 1\}^{n}\to \{0, 1\}$ such that $\pr_{(\bx,y)\sim D}[h(\bx)\neq y]\leq 1/2-1/\poly(n)$,
    where $\max(q,1/\tau,T,1/p)=2^{\Tilde{O}(n^{1/3}\alpha^{-2/3})}$. 
\end{theorem}

The high-level idea of the weak learner is similar to \Cref{alg:learner-sq},
with the main differences being that the degree of the $L_1$ regression is lower, 
and we only need to remove at most $O(1/\alpha)$ mass of the domain through 
guessing heavy coordinates.
Let $f_S(\bx)=\bigvee_{i \in S} \bx_i$ be any 
optimal monotone disjunction as before.
We set the radius parameter $r=n^{2/3}\alpha^{-1/3}$ 
and call a coordinate $i\in [n]$ heavy 
if $i\in S$ and $\E_{(\bx,y)\sim D}[\Ind(\bx_i=1)]\ge r/n$.
Given an input distribution $D$ of $(\bx,y)$, 
we will either (a) guess that some $i\in [n]$ is a heavy coordinate, 
or (b) guess that there is no heavy coordinate.
Suppose that with probability $1/2$ we sample $i\sim u([n])$ 
and guess that $i$ is heavy, and with probability $1/2$ 
we guess that there is no heavy coordinate. 
Then this guess is always correct with probability $\Omega(1/n)$. 
Suppose that our guess is correct.
\begin{enumerate}[leftmargin=*]
    \item If the algorithm guesses (a), then we can obtain a partial classifier $h(\bx)=1$ for any $\bx$ that has $\bx_i=1$ and remove these points from the domain. We can also remove coordinate $i$ from further consideration, 
    since all $\bx$ remaining have $\bx_i=0$.
    \item If the algorithm guesses (b), 
    then we can discard any coordinate $i$ such that $\E_D[\bx_i]\ge r/n$ 
    (as in \Cref{alg:learner-sq}), 
    since they cannot be in $S$ by the definition of (b). 
    After that, we get $\E_D[W_I(\bx)]\leq r$, 
    where $I$ is the set of coordinates remaining in consideration.
    By Markov's inequality, at least $1/2$ of the probability mass 
    satisfies $W_I(\bx)\leq 2 r$.
    Since we assumed that $\opt=O(1/\alpha)$, 
    this implies that the error conditioned on this $1/2$ 
    is still $O(1/\alpha)$. By \Cref{fct:l1-regression} and 
    \Cref{lem:approx-degree-coordinate}, 
    applying $L_1$-regression with degree-$cr^{1/2}\alpha^{-1/2}$ 
    (where $c$ is a sufficiently large constant) allows us 
    to learn within additive error $1/2-1/(c\alpha)$, 
    which suffices for weak learning.    
\end{enumerate}
Notice that once the algorithm guesses case (b) (and is correct), 
then we immediately get a weak learner.
However, if the algorithm guesses case (a), 
we will also be able to remove $r/n$ mass 
and obtain a partial classifier $h$ that agrees with $f_S$ on these mass. 
Since we assumed that $\opt=O(1/\alpha)$, 
this can happen at most $\opt/(r/n)=\opt/(n^{-1/3}\alpha^{-1/3})=n^{1/3}\alpha^{-2/3}$ 
times before we see a partial classifier that is non-trivially correlated 
with the labels on the mass we remove. 
This in turn gives a weak learner.
Given the weak learner in \Cref{thm:main-tradeoff-weak}, 
we can use standard boosting 
techniques~\cite{KMV:08, KalaiK09, Feldman:10ab}   
to get a $(\alpha,\eps)$-approximate agnostic learner.
The algorithm and the proofs of \Cref{thm:main-tradeoff} and \Cref{thm:main-tradeoff-weak} are deferred to Appendix~\ref{app:tradeoff}.

\section{Conclusions and Open Problems} \label{sec:concl}
In this work, we give an $2^{\tilde{O}(n^{1/3})}$ time algorithm 
for agnostically learning disjunctions, substantially improving
on the previous bound of $2^{\tilde{O}(n^{1/2})}$. As a corollary, 
we obtain the first super-polynomial separation between CSQ and SQ
in the context of agnostic learning. The obvious open question is whether significantly faster agnostic learners for disjunctions exist. 
We note that any improvement 
on the complexity of our algorithm  
would also imply a similar improvement 
on the complexity of (realizable) PAC 
learning of DNFs, which would in turn improve 
upon the previous bound of \cite{KS01}.
Finally, it is worth pointing out that 
even for the much broader class of linear threshold functions, the best known lower bounds 
(SQ and cryptographic) are only quasi-polynomial in $n$ 
(for constant $\eps$) (see~\cite{Daniely16, DKMR22, Tiegel2022}).  
Closing this large gap between known upper and lower bounds is a 
challenging direction for future work.

\bibliographystyle{alpha}
\bibliography{clean2}

\newpage

\appendix 

\section*{Appendix} 
\section{Omitted Proofs from Section~\ref{sec:non-sq-ub}} 
\label{app:exact-nonsq}

We formalize the notion of weak agnostic learning in 
the following definition.

\begin{definition} [$(\alpha,\gamma)$-weak learner] \label{def:weak-learner}
Let $\C$ be a concept class of Boolean-valued functions $f:X\to \{0, 1\}$. 
Given $\alpha,\gamma\in(0,1/2)$ where $\alpha> \gamma$ and a distribution $D$ on $X\times \{0, 1\}$ such that $\opt\eqdef\min_{f\in \C}\pr_{(\bx,y)\sim D}[f(\bx)\neq y]\leq 1/2-\alpha$, we call a hypothesis $h:X\to \{0, 1\}$ a $\gamma$-weak hypothesis if $\pr_{(\bx,y)\sim D}[h(\bx)\neq y]\leq 1/2-\gamma$.
Given i.i.d.\ samples from $D$, the goal of the learning algorithm $A$ is to output a $\gamma$-weak hypothesis with at least constant probability.
We will say that the algorithm $\A$ distribution-free $(\alpha,\gamma)$-weak agnostically learns $\C$.  
\end{definition}

Given a distribution-free $(\alpha,\gamma)$-weak agnostic learner for a concept class $\C$, it is possible to boost it to a learner that distribution-free agnostically learns $\C$ to additive error $\alpha$ as stated in the following fact (see Theorem 1.1 of \cite{Feldman:10ab}). We remark that the results from \cite{KMV:08} and \cite{KalaiK09} would also suffice for the same purpose.
\begin{fact} \label{fct:boost-additive}
    There exists an algorithm ABoost that for every concept class $\C$, given a distribution-free $(\alpha,\gamma)$-weak agnostic learning algorithm $\A$ for $\C$, distribution-free agnostically learns $\C$ to additive error $\alpha$. Furthermore, ABoost invokes $\A$ $O(\gamma^{-2} )$ times and runs in time $\poly\left (T,1/\gamma\right )$, where $T$ is the time and sample complexity of $\A$.
\end{fact}

\begin{proof} [Proof of \Cref{thm:main-intro-informal}]
Given the above setup, a direct application 
of \Cref{fct:boost-additive} with $\alpha=\eps$ 
and $\gamma=\eps/100$, 
gives a learning algorithm for distribution-free agnostic learning monotone disjunctions to error $\opt+\eps$ 
with sample and computational complexity 
$$\poly\left (2^{\Tilde{O}(n^{1/3}\log(1/\eps))},1/\eps\right)
=2^{\Tilde{O}\left (n^{1/3}\log(1/\eps)\right )} \;.$$
Since one can easily reduce learning general disjunctions (which includes negation of the variables) to learning monotone disjunctions by including negated variables as additional features, this completes the proof.
\end{proof}

\section{Omitted Proofs from \Cref{sec:tradeoff-alg}}
\label{app:tradeoff}
The algorithm establishing \Cref{thm:main-tradeoff-weak} is provided as \Cref{alg:tradeoff} below.
\begin{algorithm} 
    \caption{Trade-off Algorithm for Distribution-free Agnostic Learning Monotone Disjunctions (Weak Learner)}
    \label{alg:tradeoff}
        \textbf{Input:} $\alpha \in [64,\sqrt{n}]$ and SQ query access to a joint distribution $D$ of $(\bx,y)$ supported on $\{0, 1\}^{n}\times \{0, 1\}$, where the error of the optimal monotone disjunction $\opt\leq 1/\alpha$.
                
        \textbf{Output:} With probability at least $2^{-\Tilde{O}(n^{1/3}\alpha^{-2/3})}$, the algorithm outputs a hypothesis $h:\{0, 1\}^{n}\to \{0, 1\}$ such that $\pr_{(\bx,y)\sim D}[h(\bx)\neq y]\leq 1/2-1/\poly(n)$.

        \vspace{0.3cm}

        \begin{algorithmic} [1]
            \Statex $\triangleright$ For convenience of the algorithm description and analysis, we fix any optimal monotone disjunction $f_S(\bx)=\bigvee_{i\in S} \bx_i$, which is unknown to the algorithm. 
            \State Set $r\leftarrow n^{2/3}\alpha^{-1/3}$, $T\leftarrow cn/(\alpha r)$, where $c$ is a sufficiently large constant and initialize 
            $U_0\leftarrow \{0, 1\}^{n}$ and $I_0\leftarrow [n]$.
            \Statex $\triangleright$ $U$ and $I$ keep track of the remaining domain and coordinates.
            \For{$t=\{0,\cdots, T\}$}
                \State Define a coordinate $i$ as heavy if $i\in S$ and $\pr_{(\bx,y)\sim D} [\Ind (\bx_i=1)|\bx\in U_t]\ge r/n$.
                \State \algparbox{With probability $1/2$, 
                sample an $i\sim u(I_t)$, guess that $i$ is heavy and run \Call{RemoveHeavyCoordinate}{}.
                With the remaining $1/2$ probability, guess that there is no heavy coordinate and run \Call{$L_1$-RegressionOnLight}{}.} \label{line:tradeoff-guess}
            \State $U_{t+1}\leftarrow U_t\backslash B_t$.
            \State \label{line:tradeoff-error-estimation}\algparbox{Let $\hat{P}$ be the answer of $\stat_D(r/(100n
            ))$ for the query function                   
            \[q(\bx,y)=\Ind(\bx\in B_t\land h_t(\bx)=y)-\Ind(\bx\in B_t\land h_t(\bx)\neq y)\; .\]
            }
            \If{$\hat{P}_U\ge r/(4n)$} \label{line:tradeoff-hypothesis-check}
            \State \algparbox{
            Sample $c'\sim u(\{0,1\})$ and define $h:\{0, 1\}\to \{0, 1\}$ as 
            $h(\bx)=h'_t(\bx)$, and 
            $h(\bx)=c'$ otherwise. Return $h$.}\label{line:tradeoff-success}
            \EndIf
        \EndFor
        \end{algorithmic}
        \hrulefill
        \begin{algorithmic} [1]
            \Procedure{RemoveHeavyCoordinate}{} 
                \State $I_{t+1}\leftarrow I_t\backslash \{i\}$
                \State \algparbox{Let $B_t=\{\bx\in U_t|\bx_i=1\}$ and $h'_t:B_t\to \{0, 1\}$ be a partial classifier on $B_t$ defined as $h'_t(\bx)=1$ if $\bx\in B_t$.}
            \EndProcedure
            \vspace{3mm}
        \end{algorithmic}
        \hrulefill           
        \begin{algorithmic} [1]
            \Procedure{$L_1$-RegressionOnLight}{} 
                \State \algparbox{Let $\hat{P}_U$ and $\hat{P}_i$ (for all $i\in I$) be the answer of $\stat_D( r/(800n))$ for the query function $q_U(\bx,y)=\Ind(\bx\in U_t)$ and $q_i(\bx,y)=\Ind(\bx_i=1\land \bx\in U_t)$ respectively. \label{line:tradeoff-estimate-conditional} \Comment{$\hat{P}_i/\hat{P}_U$ will be the empirical estimate of $\pr_{(\bx,y)\sim D}[\bx_i=1\mid \bx\in U_t]$.}}
                \State \algparbox{$I_{t+1}\leftarrow I_t\backslash \{i\mid \hat{P}_i/\hat{P}_U\ge (1+1/100)r/n\}$.}
                \State \algparbox{Let $B_t=\{\bx\in U_t|W_{I_{t+1}}(\bx)\leq 2r\}$ and $D'$ be the joint distribution of $(\bx,y)\sim D$ conditioned on $\bx\in B_t$, which we have SQ query access to by asking queries on the distribution $D$.}
                \State \algparbox{Apply the degree-$\left (cr^{1/2}\alpha^{-1/2}\right )$ polynomial $L_1$-regression algorithm in \Cref{fct:l1-regression} on $D'$ to learn a hypothesis $h'_t$.}
                \Comment{ The degree of the $L_1$ regression is lower compared with \cref{alg:learner-sq} due to different $r$.}
            \EndProcedure
            \vspace{3mm}
        \end{algorithmic}
\end{algorithm}

We now give the proof for \Cref{thm:main-tradeoff-weak}.
\subsection{Proof for \Cref{thm:main-tradeoff-weak}}
\begin{proof} [Proof for \Cref{thm:main-tradeoff-weak}]
    The proof here is similar to that for \Cref{alg:learner-sq}. 
    For convenience of the analysis, fix $f_S$ be the same optimal hypothesis we fixed in the algorithm to maintain a consistent definition of heavy coordinates.
    Then notice that as we have discussed before, the guess in Line~\ref{line:tradeoff-guess} is always correct with $1/(2n)$ probability.
    With probability at least $2^{-\Tilde{O}(n^{1/3}\alpha^{-2/3})}$, all the guesses made in Line~\ref{line:tradeoff-guess} are correct,
    and it suffices for us to show that 
    given all the guesses are correct, then the algorithm outputs a hypothesis with error at most $1/2-1/\poly(n)$ with at least constant probability. For the rest of the proof, we assume that all the guesses in Line~\ref{line:tradeoff-guess} are correct.

    We first show that the algorithm always succeeds with at least a constant probability if it terminates via Line~\ref{line:tradeoff-hypothesis-check}.
    Since the ``if'' condition in Line~\ref{line:tradeoff-hypothesis-check} is satisfied, we get that with at least constant probability that
    \begin{align*}
        \pr_{(\bx,y)\sim D}[h(\bx)\neq y]=&\pr_{(\bx,y)\sim D}[h(\bx)\neq y\land \bx\in B_t]+\pr_{(\bx,y)\sim D}[h(\bx)\neq y\land \bx\not\in B_t]\\
        \le &\frac{1}{2}\pr_{(\bx,y)\sim D}[\bx\in B_t]-\Omega(r/n)+\frac{1}{2}\pr_{(\bx,y)\sim D}[\bx\not\in B_t]\\
        \le &1/2-1/\poly(n)\; ,
    \end{align*}
    where the second from the last inequality follows from that we sampled $c'\sim u(\{0,1\})$.

    Given the statement in the above paragraph, it suffices for us to prove that the algorithm will terminate via Line~\ref{line:tradeoff-hypothesis-check} deterministically if all guesses in Line~\ref{line:tradeoff-guess} are correct.
    We first give the following lemma, which is an analog of \Cref{lem:sq-algorithm-error-progress}.
    \begin{lemma} \label{fct:tradeoff-iteration-removed-mass}
        In \Cref{alg:tradeoff}, given all the guesses in Line~\ref{line:tradeoff-guess} are correct, and suppose that the algorithm did not terminate via Line~\ref{line:tradeoff-success} in the first $t$ iterations.
        Then for all the first $t$ iterations, the algorithm must have guessed that there is a heavy coordinate and $\pr_{(\bx,y)\sim D}[\bx\in U_{t+1}]\geq 1/2$.
    \end{lemma}
    \begin{proof}
        We will prove the statement by induction. The statement is trivially true for $t=-1$.
        Given the statement is true for any $i\leq t-1$, we will prove that the statement is true for $t$.
        Suppose the algorithm did not terminate in the first $t$ iterations, then since the statement is true for $i\leq t-1$, we have that all the first $t-1$ iterations must all guess that there is a heavy coordinate and $\pr_{(\bx,y)\sim D}[\bx\in U_{i}]\geq 1/2$ for all $i\leq t$.
        
        We first show that the $t$-th iteration also guesses that there is a heavy coordinate.
        By \Cref{fct:SQ-oracle-conditional}, we have that 
        $|\hat{P}_i/\hat{P}_U-\pr_{(\bx,y)\sim D}[\bx_i=1\mid \bx\in U_t]|\leq (1/100)(r/n)$,
        therefore, 
        \[\E_{(\bx,y)\sim D}[W_{I_{t+1}}(\bx)\mid \bx\in  U_{t}]\leq (4/3)r\; .\]
        Then from $\pr_{(\bx,y)\sim D}[\bx\in U_{t}]\geq 1/2$, the definition of $B_t$ and Markov's inequality, we get  $\pr_{(\bx,y)\sim D}[\bx\in B_t]\geq 1/6$.
        Assume for the purpose of contradiction that the $t$-th iteration guesses that there is no heavy coordinate.
        Then since $\opt\leq 1/\alpha$, we have $\pr_{(\bx,y)\sim D'}[f_S(\bx)\neq y]\leq 6/\alpha<1/8$.        
        Notice that any $i\not\in I_{t+1}$ must be removed by \textsc{RemoveHeavyCoordinate} of previous iterations since all previous iterations guess that there is a heavy coordinate.
        Furthermore, since \textsc{RemoveHeavyCoordinate} also removes any $\bx$ such that $\bx_i=1$ from $U$, we must have that for any $\bx\in U_t$ and $i\not \in I_{t+1}$, $\bx_i=0$.
        This implies that for any $\bx\in B_t$, $W_{S}(\bx)\le W_{I_{t+1}}(\bx)\leq 2r$. 
        Then from \Cref{lem:approx-degree-coordinate} and \Cref{fct:l1-regression}, we have that the $L_1$-regression in \textsc{$L_1$-RegressionOnLight} will learn $f_S$ to additive error $1/2-7/\alpha$ on $D'$.
        This implies that 
        \begin{align*}            
        \pr_{(\bx,y)\sim D'}[h'_t(\bx)\neq y]
        \leq &\pr_{(\bx,y)\sim D'}[f_S(\bx)\neq y]+(1/2-7/\alpha)\\
        \leq &6/\alpha+1/2-7/\alpha\leq 1/2-1/\alpha\; .
        \end{align*}
        Therefore, we have $\E_{(\bx,y)\sim D}[q(\bx,y)]\geq 2\E_{(\bx,y)\sim D}[\bx\in B_t]/\alpha\geq 1/(3\alpha)$ where $q$ is the query function in Line~\ref{line:tradeoff-error-estimation}, and $\hat{P}$ must satisfy the ``if'' condition in Line~\ref{line:tradeoff-hypothesis-check}.
        Then the algorithm will terminate in iteration $t$. 
        This contradicts the assumption, and therefore, in the $t$-th iteration, the algorithm must also guess that there is a heavy coordinate.
        
        Now it only remains to show $\pr_{(\bx,y)\sim D}[\bx\in U_{t+1}]\geq 1/2$.
        Given the algorithm guesses that there is a heavy coordinate for the first $t$ iterations, we assume for the purpose of contradiction that $\pr_{(\bx,y)\sim D}[\bx\in U_{t+1}]<1/2$.
        Then using the fact that the ``if'' condition in Line~\ref{line:tradeoff-hypothesis-check} is never satisfied for the first $t$ iterations, we have that
        \begin{align*}  
            &\pr_{(\bx,y)\sim D}[f_S(\bx)\neq y]\\
            \ge &\sum_{i\in [t]}\pr_{(\bx,y)\sim D} [f_S(\bx)\neq y\land \bx\in B_i ]\\
            = &\sum_{i\in [t]}\pr_{(\bx,y)\sim D} [h'_t(\bx)\neq y\land \bx\in B_i ]\\
            = &\sum_{i\in [t]}\pr_{(\bx,y)\sim D} [h'_t(\bx)\neq y|\bx\in B_i]\pr_{(\bx,y)\sim D}[\bx\in B_i ]\\
            \ge &\min_i(\pr_{(\bx,y)\sim D} [h'_t(\bx)\neq y|\bx\in B_i])\sum_{i\in [t]}\pr_{(\bx,y)\sim D}[\bx\in B_i ]\\
            \ge &\min_i(\pr_{(\bx,y)\sim D} [f_S(\bx)\neq y|\bx\in B_i])/2\; .
        \end{align*}
        Notice that from the definition of heavy coordinates, we have 
        \[\pr_{(\bx,y)\sim D}[\bx\in B_i]\ge \pr_{(\bx,y)\sim D}[\bx\in U_i](r/n)=(1/2)(r/n)\; .\]
        Then, 
        \begin{align*}
            &\pr_{(\bx,y)\sim D}[h'_t(\bx)\neq y\mid \bx\in B_i]\\
            =&\pr_{(\bx,y)\sim D}[h'_t(\bx)\neq y\land \bx\in B_i]/\pr_{(\bx,y)\sim D}[\bx\in B_i]\\
            =&\frac{\pr_{(\bx,y)\sim D}[\bx\in B_i]-(\pr_{(\bx,y)\sim D}[h'_t(\bx)= y\land \bx\in B_i]-\pr_{(\bx,y)\sim D}[h'_t(\bx)\neq y\land \bx\in B_i])}{2\pr_{(\bx,y)\sim D}[\bx\in B_i]}\\
            \geq & 1/2-\frac{(1/4)(r/n)+(1/100)(r/n)}{r/n}\geq 1/5\; .&
        \end{align*}
        Therefore, plugging it back gives 
        $\pr_{(\bx,y)\sim D}[f_S(\bx)\neq y]\geq 1/10$.
        This contradicts the assumption that $\opt\leq 1/\alpha\leq 1/64$ and therefore, we must have 
        ${\pr_{(\bx,y)\sim D}[\bx\in U_{t+1}]}\geq 1/2$.
        This completes the proof.
    \end{proof}

    \Cref{fct:tradeoff-iteration-removed-mass} implies that once the algorithm guesses that there is no heavy coordinate, the algorithm must terminate via Line~\ref{line:tradeoff-hypothesis-check}. Therefore, we only need to show that the algorithm cannot keep guessing that there is a heavy coordinate for $T$ iterations.
    Notice that since all guesses are correct and all guesses are that there is a heavy coordinate, from the definition of heavy coordinates, we have $h_i(\bx)=f_S(\bx)$ for any $\bx\in B_i$ for all iteration $i$.
    Given the algorithm does not terminate for $T$ iterations and the ``if'' condition in Line~\ref{line:tradeoff-hypothesis-check} is never satisfied for these $T$ iterations, we would have
    \begin{align*}    
    \pr_{(\bx,y)\sim D}[f_S(\bx)\neq y]
    \ge &\sum_{i=1}^T\pr_{(\bx,y)\sim D}[h_i(\bx)\neq y\land \bx\in B_i]\\
    \ge &\sum_{i=1}^T (\pr_{(\bx,y)\sim D}[\bx\in B_i]\\
    &-(\pr_{(\bx,y)\sim D}[h_i(\bx)= y\land \bx\in B_i]-\pr_{(\bx,y)\sim D}[h_i(\bx)\neq y\land \bx\in B_i])  )\\
    \ge &\sum_{i=1}^T(r/(2n)-r/(3n))
    =\Omega\left (Tr/n\right )=c/\alpha\; ,
    \end{align*}
    where $c$ is a sufficiently large constant.
    This contradicts the assumption that $\opt\leq 1/\alpha$.
    This completes the proof that given all the guesses in Line~\ref{line:tradeoff-guess} are correct, \Cref{alg:tradeoff} will, with at least constant probability, outputs a hypothesis $h$ such that $\pr_{(\bx,y)\sim D}[h(\bx)\neq y]\leq 1/2-1/\poly(n)$.

    It only remains to verify the query and computational complexity. Notice that the smallest query tolerance that the algorithm directly asked is at least $1/\poly(n)$, and the smallest query tolerance asked by the $L_1$-regression is at least $d^{-cr^{1/2}\alpha^{-1/2}}=2^{-\Tilde{O}(n^{1/3}\alpha^{-2/3})}$. Furthermore, the computational complexity of the algorithm is $Tn^{O(r^{1/2}\alpha^{-1/2})}=2^{\Tilde{O}(n^{1/3}\alpha^{-2/3})}$, and the total number of queries the algorithm asks must be bounded by the same quantity.
    This completes the proof for \Cref{alg:tradeoff}.
\end{proof}

\subsection{Proof of \Cref{thm:main-tradeoff}}

    Given \Cref{thm:main-tradeoff-weak}, we are now ready to prove \Cref{thm:main-tradeoff}.
    We will first need the following fact about boosting for agnostic learning to multiplicative error from \cite{Feldman:10ab}.
    \begin{fact} \label{fct:boost-approximate}
        There exists an algorithm ABoostDI that for every concept class $C$ over $X$, given a distribution independent $(\alpha,\gamma)$-weak agnostic learning algorithm $A'$ for $C$, for every distribution $A=(D,f)$ over $X$ and $\eps>0$, produces a hypothesis $h$ such that $\pr_{(\bx,y)\sim D}[f(\bx)\neq y]\leq \opt/(1-2\alpha)+\eps$. Further, ABoostDI invokes $A'$ $O(\gamma^{-2}\Delta^{-1}\log(1/\Delta))$ times for $\Delta=\opt/(1-2\alpha)$ and runs in time $\poly(T,1/\gamma,1/\eps)$, where $T$ is the running times of $A'$.
    \end{fact}
    We then prove \Cref{thm:main-tradeoff}.
    \begin{proof} [Proof of \Cref{thm:main-tradeoff}]
    Given \cref{thm:main-tradeoff-weak}, to $(\alpha',\eps)$-approximate agnostically learn disjunctions, we first, without loss of generality, assume that $\alpha'\opt\geq\eps$ (if this is not true, take a new $\eps'=\eps/4$ and $\alpha''$ such that $\eps/4\le\alpha''\opt\le\eps/2$).

    Then we can simply apply \Cref{fct:boost-approximate} and take the $\alpha$ in \Cref{fct:boost-approximate} as $(1/2-\alpha'/2)$ and the boosting algorithm ABoostDI invokes \Cref{alg:tradeoff} at most $\poly(n/\eps)$ times and its own runtime is $\poly(T,1/\gamma,1/\eps)=2^{\Tilde{O}(n^{1/3}\alpha^{-2/3})}\poly(1/\eps)$. This gives the algorithm for agnostic learning monotone disjunctions to error $\alpha'\opt+\eps$.
    Notice that this algorithm is still SQ since the boosting algorithm does not require sample access to the distribution.
    As we have argued before,
    one can easily reduce learning general disjunctions (which includes negation of the variables) to learning monotone disjunctions by including negated variables as additional features. This completes the proof.
    \end{proof}

\section{CSQ Complexity of Distribution-free Agnostic Learning Disjunctions} 
\label{sec:csq-lb}
In this section, we characterize the complexity of distribution-free agnostic learning disjunctions in the \emph{Correlational Statistical Query} (CSQ) model.

\paragraph{Basics on CSQ Model}
In the context of \Cref{def:sq}, 
a \emph{Correlational Statistical Query} is 
a bounded function $q:X \rightarrow [-1,1]$. 
We define $\mathrm{CSTAT}(\tau)$ to be the oracle 
that, given any such query $q$, outputs a value $v\in [-1,1]$ 
such that $|v-\E_{(\bx,y)\sim D}[(2y-1)q(\bx)]|\leq\tau$, 
where $\tau>0$ is the \emph{tolerance} parameter of the query.
A \emph{Correlational Statistical Query (CSQ)) algorithm} 
is an algorithm whose objective is to learn some information about an unknown 
distribution $D$ by making adaptive calls to the corresponding 
$\mathrm{CSTAT}(\tau)$ oracle.
The query complexity of a CSQ algorithm is defined as $m/\tau^2$, where $m$ is the number of queries and $\tau$ is the smallest tolerance of queries the algorithm calls to the corresponding 
$\mathrm{CSTAT}(\tau)$ oracle.

It is well-known that CSQ queries are a special case of SQ queries, 
and, therefore, have weaker power.
In particular, any SQ query function $q_{\mathrm{sq}}:X\times \{0,1\}\to [-1,1]$ 
can always be decomposed to $q_{\mathrm{sq}}(\bx,y)=q_1(\bx,y)+q_2(\bx,y)$, 
where $q_1(\bx,y)=(q_{\mathrm{sq}}(\bx,0)+q_{\mathrm{sq}}(\bx,1))/2$ 
is a query function independent of the label $y$, 
and $q_2(\bx,y)=(2y-1)(-q_{\mathrm{sq}}(\bx,0)+q_{\mathrm{sq}}(\bx,1))/2$ 
is equivalent to a CSQ query.
An intuitive interpretation is that, compared with the SQ model, 
the CSQ model loses exactly the power to make label-independent queries 
about the distribution, i.e., the power to ask queries about the marginal 
distribution of $\bx$.

\medskip

\paragraph{CSQ Upper Bound on Weak Agnostic Learning}
We note that there is a CSQ weak agnostic learner with $2^{\tilde{O}(n^{1/2}\log(1/\eps))}$ time and query complexity that 
outputs a hypothesis with error $1/2-\Omega(\eps)$ 
given that $\opt\leq 1/2-\eps$.
\begin{fact} \label{fact:csq-eps-weak-learning}
    Let $D$ be an unknown distribution supported on 
    $\{0, 1\}^{n}\times \{0, 1\}$ and $\eps\in (0,1/2)$. 
    Suppose there is a monotone disjunction 
    $f: \{0, 1\}^{n}\to \{0, 1\}$ such that 
    $\pr_{(\bx,y)\sim D}[f(\bx)\neq y]\leq\ 1/2-\eps$.
    There is an algorithm that makes at most $q$ queries to $\mathrm{CSTAT}_D(\tau)$, and deterministically 
    returns a hypothesis $h:\{0, 1\}^{n}\to \{0, 1\}$ such that $\pr_{(\bx,y)\sim D}[h(\bx)\neq y]\leq 1/2-\Omega(\eps)$, where $\max(q,1/\tau)=2^{\Tilde{O}(n^{1/2}\log(1/\eps))}$.
\end{fact}

\begin{proof} 
    For convenience of the proof, we will use $\{-1,1\}$ for Boolean values instead of $\{0,1\}$ for this proof.
    We start by proving the following fact.
    \begin{fact}\label{fct:eps-correlation-polynomial}
    Let $D$ be an unknown distribution supported on 
    $\{-1, 1\}^{n}\times \{-1, 1\}$ and $\eps\in (0,1/2)$. 
    Suppose there is a monotone disjunction 
    $f: \{-1, 1\}^{n}\to \{-1, 1\}$ such that 
    $\pr_{(\bx,y)\sim D}[f(\bx)\neq y]\leq\ 1/2-\eps$.
    Then there is a polynomial $p$ of degree at most $O(n^{1/2}\log(1/\eps))$ such that 
    for all $\bx\in \{-1,1\}^n$, $p(\bx)\in [-1,1]$ and 
    $\E_{(\bx,y)\sim D}[yp(
    \bx)]=\Omega(\eps)$.
    \end{fact}
    \begin{proof}
        The proof here directly follows from \Cref{lem:approx-degree-coordinate}.
        Let $p'$ be the degree-$O(n^{1/2}\log(1/\eps))$ polynomial that is a $c\eps$-approximate (for sufficiently small constant $c$) polynomial of $f$ in \Cref{lem:approx-degree-coordinate} with $r=n$.
        Then let $p$ be defined as $p(\bx)=p'(\bx)/(1+\eps)$, which is a $2c\eps$-approximate polynomial of $f$.
  It is easy to see that for all $\bx\in \{-1,1\}^n$, $p(\bx)\in [-1,1]$ from its definition.
        For the correlation, we have
        \[
        \E_{(\bx,y)\sim D}[yp(\bx)]\geq (1-2c\eps)\E_{(\bx,y)\sim D}[y=f(\bx)]-\E_{(\bx,y)\sim D}[y\neq f(\bx)]=\Omega(\eps)\; .
        \]
        This completes the proof.
    \end{proof}
    We then show that it is always possible to find such a polynomial in \Cref{fct:eps-correlation-polynomial} using CSQ queries with query complexity at most $2^{\Tilde{O}(n^{1/2}\log(1/\eps))}$.
    We define a parity function over the set $S\subseteq [n]$ as $g_S(\bx)=\bigoplus_{i\in S} \bx_i$, where $\bigoplus$ is the exclusive or operator.
    Notice that parity functions over sets of size at most $d$ spans any polynomials $p$ of degree at most $d$ over $\{-1,1\}^n$, i.e., $p(\bx)=\sum_{\left \{S\subseteq [n]\land  |S|\leq d\right \} }\alpha_S g_S(\bx)$ for some $\alpha_S\in \R$.
    Furthermore, since parities form an orthonormal basis on $u(\{-1,1\})^n$, 
    the parameters $\alpha_S$ satisfy $|\alpha_S|=|\E_{\bx\sim u(\{-1,1\}^n)}[p(\bx)g_S(\bx)]|\leq 1$.
    Notice that 
    \[\E_{(\bx,y)\sim D}[y p(\bx)]=\sum_{S\subseteq [n]\land |S|\leq d}\alpha_S\E_{(\bx,y)\sim D}[y g_S(\bx)].\]
    Therefore, if we already know the value of $\E_{(\bx,y)\sim D}[y g_S(\bx)]$ (up to $2^{-\Tilde{O}(n^{1/2}\log(1/\eps))}$ error) from the CSQ oracle, we can approximately calculate the value of $\E_{(\bx,y)\sim D}[yp(\bx)]$ (up to $o(\eps)$ error) with no additional queries. This suffices for us to find such a $p$ in \Cref{fct:eps-correlation-polynomial}.
    Namely, let $\hat{p}_S$ be the answer of the CSQ oracle for the parity function $g_S$ with error tolerance $2^{-\Tilde{O}(n^{1/2}\log(1/\eps))}$ (with sufficiently large implied constant), i.e.,
    \[
    |\hat{p}_S-\E_{(\bx,y)\sim D}[y g_S(\bx)]|\leq \tau\; ,
    \]
    where $\tau=2^{-c(n^{1/2}\log(1/\eps))\log(n^{1/2}\log(1/\eps))^c}$ and $c$ is a sufficiently large constant.
    Then consider the following LP for finding the polynomial $p(\bx)=\sum_{S\subseteq [n]\land |S|\leq d} \alpha_S g_S(\bx)$:
    \begin{align*}
    &\mathrm{max} \quad \sum_{S\subseteq [n]\land |S|\leq d} \alpha_S \hat{p}_S \\[2pt] 
    &\mathrm{s.t.} \quad 
    \begin{array}[t]{r r l l l}
         \sum_{S\subseteq [n]\land |S|\leq d} \alpha_S g_S(\bx) &\in  [-1,1], \,&\forall& \bx \in \{-1,1\}^n\; ,\\
           \alpha_S &\in [-1,1], \,&\forall&  S\subseteq [n]\land |S|\leq d\; .
     \end{array}
    \end{align*}
    By \Cref{fct:eps-correlation-polynomial}, the optimal solution of the LP must be $\Omega(\eps)$.
    Furthermore, let $\alpha_S$ be any optimal solution to the LP and let $p(\bx)= \sum_{S\subseteq [n]\land |S|\leq d} \alpha_S p(\bx)$. Then we have
    \begin{align*}
        \E_{(\bx,y)\sim D}[y p(\bx)]=&\sum_{S\subseteq [n]\land |S|\leq d}\alpha_S \E_{(\bx,y)\sim D}[y g_S(\bx)]
        \geq &\sum_{S\subseteq [n]\land |S|\leq d} \alpha_S \left (\hat{p}_S-2^{-\Tilde{O}(n^{1/2}\log(1/\eps))}\right )
        =\Omega(\eps)\; .
    \end{align*}
    It only remains for us to show how to get a hypothesis with error $1/2-\Omega(\eps)$ from such a $p$.

    Notice that given $p(\bx)\in [-1,1]$ for all $\bx$ and $\E_{(\bx,y)\sim D}[y p(\bx)]=\Omega(\eps)$,
    we have that the $L_1$ loss of $p$ is 
    \[
    \E_{(\bx,y)\sim D}[|p(\bx)-y|]=\E_{(\bx,y)\sim D}[1-(y p(\bx))]=1-\Omega(\eps)\; .
    \]
    To convert the $L_1$ loss to the 0-1 loss of the output hypothesis, we first discretize the interval $[-1,1]$.
    Let $T\eqdef\{0,c\eps,2c\eps,\cdots,\lceil 2/\eps\rceil c\eps\}$, where $|T|\geq 2/(c\eps)-1$ and $c$ is a sufficiently small constant.
    Let $t\sim u(T)$, and define the corresponding random hypothesis $h_t$ as $h(\bx)=\sign(p(\bx)-t)$.
    Then notice that the expected 0-1 loss of $h_t$ is
    \begin{align*}
        \E_{t\sim u(T)}\left [\E_{(\bx,y)\sim D}[h_t(\bx)\neq y]\right ]=&\E_{(\bx,y)\sim D}\left [\E_{t\sim u(T)}[h_t(\bx)\neq y]\right ]\\
        =&\E_{(\bx,y)\sim D}\left [\E_{t\sim u(T)}[t\in [p(\bx),y]\cup [y,p(\bx)]]\right ]\\
        \leq &\E_{(\bx,y)\sim D}\left [\frac{|p(\bx)-y|/(c\eps)+1}{2/(c\eps)-1}\right ]\\
        \leq &\E_{(\bx,y)\sim D}\left [\frac{|p(\bx)-y|+2c\eps}{2}\right ]\\
        =& \E_{(\bx,y)\sim D}[|p(\bx)-y|]/2+c\eps\leq 1/2-\Omega(\eps)\; .
    \end{align*}
    Therefore, there must be a $t\in T$ such that $\E_{(\bx,y)\sim D}[h_t(\bx)\neq y]\leq 1/2-\Omega(\eps)$.
    Notice that we can query the value of $\E_{(\bx,y)\sim D}[h_t(\bx)\neq y]$ using CSQ queries as
    \[\E_{(\bx,y)\sim D}[h_t(\bx)\neq y]=1/2(1-(\E_{(\bx,y)\sim D}[h_t(\bx)= y]-\E_{(\bx,y)\sim D}[h_t(\bx)\neq y]))=1/2(1-\E_{(\bx,y)\sim D}[yh(\bx)])\; .\]
    Therefore, we can simply check $\E_{(\bx,y)\sim D}[h_t(\bx)\neq y]$ for all $t\in T$ using CSQ queries with $c\eps$ error tolerance for a sufficiently small constant $c$, and then output the $h_t$ with the smallest error.
\end{proof}

\paragraph{CSQ Lower Bound on Weak Agnostic Learning}

The following fact about CSQ lower bound for distribution-free agnostic learning disjunctions is given in \cite{GKK20}. 
\begin{fact} \label{fct:csq-lb-main}
    Any CSQ algorithm for distribution-free agnostic learning disjunctions on $\{0, 1\}^n$ to error $\opt+1/100$ either requires a query of tolerance $2^{-\Omega(n^{1/2})}$ or $2^{\Omega(n^{1/2})}$ queries.  
\end{fact}

We include the proof here for completeness.
The CSQ lower bound here follows from the approximate degree of disjunctions.
The following definition and facts from \cite{GKK20} state the relation between CSQ lower bounds and the approximate degree. We first give the definition of pattern restriction from the pattern matrix method in \cite{sh08}.

\begin{definition}[Pattern restrictions] \label{def:pattern-restriction}
Let $C=\bigcup_{n\in \N} C_n$ be the union of some classes $C_n$ of Boolean-valued function on $\{0, 1\}^n$.
We say $C$ is closed under pattern matric restriction if for any $k$, $n$ that is a multiple of $k$, and any $f\in C_k$, the function $\bx\mapsto f(\bx_V\oplus \bw)$ on $\{0, 1\}^n$ lies in $C_n$ for any $V\subseteq [n]$ of size $k$ and $\bw\in \{0, 1\}^k$. In the common case where $n$ is a small constant multiple of $k$, we will often be somewhat loose and not explicitly distinguish between $C_k$ and $C_n$ and just refer to $C$. Indeed, one can consider $C_k$ to effectively be a subset of $C_n$ using only some $k$ out of $n$ bits.  
\end{definition}

\begin{fact} [Theorem 1.2 of \cite{GKK20}] \label{fct:approximate-degree-bound}
    Let $C$ be a Boolean-valued function class close under pattern restriction (\Cref{def:pattern-restriction}), with $1/2$-approximate degree $\Omega(d)$. Any distribution-free agnostic learner for $C$ using only correlational statistical queries of tolerance $\tau\leq 1/10$ requires at least $2^{\Omega(d)}\tau^2$ queries in order to agnostically learn $C$ up to excess error $1/100$, i.e., true error $\opt+1/100$.
\end{fact}

We will combine the above fact with the following lower bound on the approximate degree of disjunctions.

\begin{fact} [see, e.g., Theorem 23 of \cite{BT22}] \label{fct:approximate-degree-to-csq}
The $1/2$-approximate degree of disjunctions is $\Omega(\sqrt{n})$.    
\end{fact}

\noindent Combining the above gives the CSQ lower bound.

\begin{proof} [Proof for \Cref{fct:csq-lb-main}]
    This follows from combining \Cref{fct:approximate-degree-bound} and \Cref{fct:approximate-degree-to-csq}.
    We take $k=cn$ and $\tau=2^{cn^{1/2}}$ where $c$ is a sufficiently small constant.
    Then \Cref{fct:approximate-degree-to-csq} implies that the function class of disjunctions on $\{0, 1\}^k$ has an approximate degree of $\Omega(n^{1/2})$.
    Given this, an application of \Cref{fct:approximate-degree-bound} proves the statement.
\end{proof}

\end{document}